\documentclass[runningheads]{llncs}
\usepackage{graphicx}

\newif\ifdraft
\draftfalse
\newif\ifmain
\maintrue
\newif\ifsupl
\supltrue

\usepackage{amsmath}
\usepackage{amssymb}
\usepackage{upgreek}
\usepackage{booktabs}
\usepackage{enumitem}
\usepackage{algorithm,algpseudocode}
\usepackage{ifthen}
\usepackage{xcolor}
\usepackage{cite}
\usepackage{tabularray}
\UseTblrLibrary{booktabs,siunitx}

\allowdisplaybreaks

\usepackage{siunitx}
\usepackage{confnames}

\definecolor{light-red}{rgb}{1,0.8,0.85}
\definecolor{light-green}{rgb}{0.85,1,0.8}
\definecolor{light-blue}{rgb}{0.8,0.85,1}

\NewTableCommand{\first}{\SetCell{light-red}}
\NewTableCommand{\second}{\SetCell{light-green}}
\NewTableCommand{\third}{\SetCell{light-blue}}
\NewTableCommand{\bold}{\SetCell{font=\bfseries}}
\NewColumnType{M}[1][]{Q[si={table-format=4.4,round-mode=places,round-precision=4},co=1,#1]}
\NewColumnType{N}[1][]{Q[si={table-format=2.4,round-mode=places,round-precision=4},co=1,#1]}

\DeclareMathOperator*{\argmax}{argmax}
\DeclareMathOperator*{\argmin}{argmin}

\makeatletter
\renewcommand\paragraph{\@startsection{paragraph}{4}{\z@}%
                        {-0.7ex \@plus -0.2ex \@minus -0.1ex}%
                        {-0.5em \@plus -0.2em \@minus -0.1em}%
                        {\normalfont\normalsize\bfseries}}
\let\oldparagraph\paragraph
\renewcommand\paragraph[1]{\oldparagraph{#1.}}

\makeatother

\newcommand{\diff}{\mathrm{d}}

\usepackage[pagebackref,breaklinks,colorlinks]{hyperref}
\usepackage[capitalize]{cleveref}

\newcommand{\tablefontsize}{\fontsize{5.5pt}{6.5pt}\selectfont}

\crefname{section}{Sec.}{Secs.}
\Crefname{section}{Section}{Sections}
\crefname{table}{Tab.}{Tabs.}
\Crefname{table}{Table}{Tables}
\crefname{theorem}{Thm.}{Thms.}
\Crefname{theorem}{Theorem}{Theorems}
\crefname{proposition}{Prop.}{Props.}
\Crefname{proposition}{Proposition}{Propositions}

\usepackage[accsupp]{axessibility} % Improves PDF readability for those with disabilities.

\ifdraft
 \usepackage{ruler}
\fi
\usepackage[width=122mm,left=12mm,paperwidth=146mm,height=193mm,top=12mm,paperheight=217mm]{geometry}

\begin{document}
\pagestyle{headings}
\mainmatter
\def\ACCVSubNumber{776} % Insert your submission number here

\title{Deep Point-to-Plane Registration by Efficient Backpropagation for Error Minimizing Function} % Replace with your title

\ifdraft
  \titlerunning{ACCV-22 submission ID \ACCVSubNumber}
  \authorrunning{ACCV-22 submission ID \ACCVSubNumber}
  \author{Anonymous ACCV submission}
  \institute{Paper ID \ACCVSubNumber}
\else
  \titlerunning{Deep Point-to-Plane Registration}
  \author{Tatsuya Yatagawa%\orcidID{0000-0003-4653-2435}
    \and Yutaka Ohtake \and Hiromasa Suzuki}
  \authorrunning{T. Yatagawa et al.}
  \institute{School of Engineering, The University of Tokyo, Tokyo, Japan
    \email{\{tatsy,ohtake,suzuki\}@den.t.u-tokyo.ac.jp}}
\fi

\ifmain

%******************
\maketitle

\begin{abstract}
  Traditional algorithms of point set registration minimizing point-to-plane distances often achieve a better estimation of rigid transformation than those minimizing point-to-point distances. Nevertheless, recent deep-learning-based methods minimize the point-to-point distances. In contrast to these methods, this paper proposes the first deep-learning-based approach to point-to-plane registration. A challenging part of this problem is that a typical solution for point-to-plane registration requires an iterative process of accumulating small transformations obtained by minimizing a linearized energy function. The iteration significantly increases the size of the computation graph needed for backpropagation and can slow down both forward and backward network evaluations. To solve this problem, we consider the estimated rigid transformation as a function of input point clouds and derive its analytic gradients using the implicit function theorem. The analytic gradient that we introduce is independent of how the error minimizing function (i.e., the rigid transformation) is obtained, thus allowing us to calculate both the rigid transformation and its gradient efficiently. We implement the proposed point-to-plane registration module over several previous methods that minimize point-to-point distances and demonstrate that the extensions outperform the base methods even with point clouds with noise and low-quality point normals estimated with local point distributions.
\end{abstract}

\section{Introduction}
\label{sec:intro}

\begin{figure}[t]
  \centering
  \includegraphics[width=\linewidth]{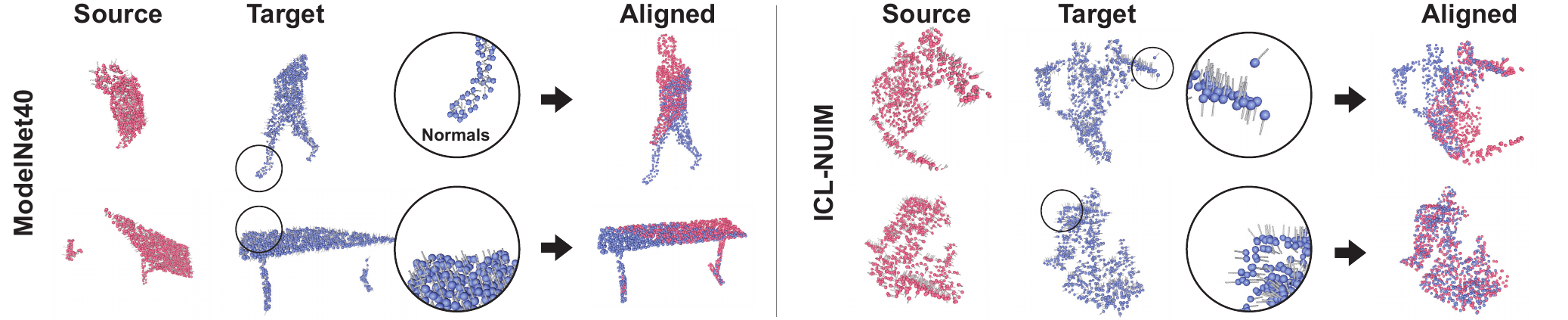}
  \caption{Our point-to-plane registration module with efficient backpropagation improves previous deep-learning-based point set registration methods. Not only for point clouds with correct point normals (left: ModelNet40), our method performs well for those with estimated normals, which may be randomly flipped (right: ICL-NUIM).}
  \label{fig:teaser}
\end{figure}

Point set registration is a fundamental problem in computer vision and graphics, and is the core of many applications, such as 3D object matching, simultaneous localization and mapping (SLAM), and dense 3D reconstruction. The iterative closest point (ICP)~\cite{besl1992method,chen1991object,segal2009generalized,yang2013goicp,rusinkiewicz2019symmetric} is one of the most popular algorithms for the registration. The standard ICP algorithm defines the correspondences by associating each point in a source point cloud with its nearest neighbor in the target point cloud. Then, it obtains a rigid transformation to minimize the distances between corresponding points. However, the traditional ICP algorithms using the distance-based point correspondences suffer from the non-convexity of the objective function and often result in a wrong registration, depending on the initial postures of input point clouds.

Recently, deep learning techniques on point clouds have provided purpose-specific point feature descriptors and have also been leveraged to define better point correspondences~\cite{aoki2019pointnetlk,wang2019deep,sarode2019pcrnet} for the registration, where the point-to-point correspondences are defined based on the vicinity of point features obtained by a neural network. In addition, several studies have been conducted to find better point-to-point correspondences between incomplete point clouds~\cite{wang2019prnet,yew2020rpmnet} or to apply deep learning to more advanced registration algorithms~\cite{yuan2020deepgmr,yew2020rpmnet}. However, most of these methods are based on the Procrustes algorithm~\cite{gower1975generalized}, which uses singular value decomposition (SVD) to minimize the sum of pairwise point-to-point distances. The differentiable SVD layer, which is provided by most machine learning libraries such as PyTorch and TensorFlow, is leveraged by the previous methods to train the network in an end-to-end manner. However, in the context of traditional point set registration, an algorithm minimizing point-to-\textit{plane} distances rather than point-to-\textit{point} distances often obtains a better rigid transformation~\cite{rusinkiewicz2001efficient}.

This fact motivated us to investigate a deep-learning technique for point-to-plane registration. A challenging part of using point-to-plane registration with deep learning is that a typical solution for the point-to-plane registration is based on an iterative accumulation of transformation matrices by minimizing a linearly approximated energy function~\cite{chen1991object}. The iterative process slows down the training process because it significantly increases the size of the computation graph needed to compute the gradients of network parameters via auto-differentiation. To address this problem, we focus on the gradient of a function defined by error minimization, which appears in point-to-plane and other standard registration algorithms. The proposed method applies the implicit function theorem to derive the analytic gradients of the rigid transformation with respect to positions and normals of input points. Recently, a similar idea has been proposed to solve the perspective-n-points (PnP) problem~\cite{chen2020end}, but we focus more on the solution to the point-to-plane registration problem defined as constrained error minimization.

This paper has three major contributions. First, we introduce a proposition to obtain a necessary condition that an analytic gradient exists for a function defined by constrained error minimization. We will show the proof of the proposition using the implicit function theorem. Second, we apply the proposition to calculate the gradients practically for the point-to-plane registration problem, which is typically formulated as constrained optimization. The analytic gradient given by our method is agnostic to how the error minimizing function is obtained. Therefore, the large computation graph is unnecessary to compute the gradient of output rigid transformation with respect to input point data. Third, we plug our point-to-plane registration module into previous methods, which originally minimized point-to-point distances. This experiment demonstrates that our extensions outperform the base methods on the evaluations with standard point cloud datasets, even with low-quality point normals estimated from local point distributions.

\section{Related Work}
\label{sec:related-work}

\paragraph{Traditional point set registration}

The classic ICP algorithms obtain a good rigid transformation to align two point clouds by alternating between finding point correspondences and calculating the rigid transformation that minimizes point-to-point distances~\cite{besl1992method}, point-to-plane distances~\cite{chen1991object,rusinkiewicz2001efficient}, and others~\cite{segal2009generalized,rusinkiewicz2019symmetric}. The vicinity of two point cloud is defined traditionally by using distances of points, but also by using probability distributions to handle noisy and partial point clouds~\cite{gabriel2016point,jian2011robust,eckart2018hgmr}. A common problem of the traditional ICP is the non-convexity of the objective function, which makes it difficult to find appropriate point correspondences. Go-ICP~\cite{yang2013goicp} uses the branch-and-bound method to find a globally optimal rotation matrix in $SO(3)$ but is several orders of magnitude slower than standard ICP algorithms even with the help of local ICP to accelerate the search process. Other approaches have solved the constrained nonlinear problem with more sophisticated solvers~\cite{fitzgibbon2003robust,maron2016point,izatt2020globally,yang2019polynomial} (e.g., Levenberg--Marquardt method).
%~\cite{fitzgibbon2003robust}, convex relaxation~\cite{maron2016point}, mixed integer programming~\cite{izatt2020globally}, and semidefinite programming~\cite{yang2019polynomial}.
However, these methods based on distances of points can often suffer from input point clouds with significantly different initial postures.

\paragraph{Feature-based point set registration}

Rather than using point positions, various previous studies have leveraged feature vectors for an input point cloud to resolve the problem of bad initial postures. Traditional hand-crafted features summarize local geometric properties, which are often described by positions and normals, into a histogram. For instance, the variety of normal orientations~\cite{tombari2010shot} and geometric relationship between point and normal pairs~\cite{rusu2008aligning,rusu2009fpfh,drost2010model,choi2012voting} have been employed. Recently, such traditional feature extractors have been replaced by deep convolutional neural networks (CNNs)~\cite{zeng20173dmatch}. For instance, PPFNet~\cite{deng2018ppfnet} and PPF-FoldNet~\cite{deng2018ppf} employ PointNet~\cite{qi2017pointnet} and FoldingNet~\cite{yang2018foldingnet} (i.e., the CNNs on a point cloud), respectively. Some other methods apply a 3D fully-convolutional network to small 3D volumes, which are defined for local point subsets~\cite{gojcic2019perfect,choy2019fully}.

\paragraph{End-to-end learning of point set registration}

End-to-end deep learning frameworks are another trend in point set registration. PointNetLK~\cite{aoki2019pointnetlk} introduces a numerically differentiable Lucas--Kanade algorithm for minimizing the distance of point features obtained by PointNet~\cite{qi2017pointnet}. An analytically differentiable variant~\cite{li2021pointnetlk} has also been developed recently. PCRNet~\cite{sarode2019pcrnet} encodes input point clouds using PointNet but directly estimates the rigid transformation between them by the network itself. Deep closest point (DCP)~\cite{wang2019deep} applies attention-based soft correspondences between input point clouds, which is later extended for partial point clouds by PRNet~\cite{wang2019prnet}. Following these seminal approaches, several studies have extended traditional registration algorithms using neural networks. For example, RPMNet~\cite{yew2020rpmnet} has extended robust point matching~\cite{gold1998rpm} to align partially overlapping point clouds. DeepGMR~\cite{yuan2020deepgmr} has extended Gaussian mixture registration~\cite{jian2011robust} for memory-efficient registration. Furthermore, a couple of studies have emphasized more reliable point pairs by weighting more heavily with an additional CNN~\cite{choy2020deep}, supervising the extraction of overlapping regions~\cite{huang2021predator}, and extracting them with Hough voting~\cite{lee2021deep}.

%------------------------------------------------------------------------
\section{Deep point-to-plane registration}
\label{sec:plane-icp}

We assume that point clouds $\mathcal{X}$ and $\mathcal{Y}$ have point normals. Hence, we denote $\mathcal{X} = \{ (\mathbf{x}_i, \mathbf{m}_i )\}_{i=1}^N$ and $\mathcal{Y} = \{ (\mathbf{y}_j, \mathbf{n}_j) \}_{j=1}^M$, where $\mathbf{x}_i$ and $\mathbf{y}_j$ denotes $i$th and $j$th vertex positions, $\mathbf{m}_i$ and $\mathbf{n}_j$ are their point normals, and $N$ and $M$ are the number of points in $\mathcal{X}$ and $\mathcal{Y}$, respectively. For simplicity, we assume points in $\mathcal{X}$ and $\mathcal{Y}$ are already paired. As shown later, our method obtains the point correspondences by a neural network as the previous studies do. Then, we align $(\mathbf{x}_i, \mathbf{m}_i)$ and $(\mathbf{y}_i, \mathbf{n}_i)$ for $i = 1, \ldots, N$ by minimizing point-to-plane distances.

\subsection{Point-to-plane registration by linear approximation}
\label{ssec:linearized-rotation}

In the point-to-plane registration~\cite{rusinkiewicz2001efficient}, the error function $E(\mathbf{R}, \mathbf{t})$ for rigid transformation $(\mathbf{R}, \mathbf{t})$ is defined for corresponding point pairs $\{ (\mathbf{x}_i, (\mathbf{y}_i, \mathbf{n}_i)) \}_{i = 1}^N$:
\begin{equation}
  E(\mathbf{R}, \mathbf{t}) = \sum_{i=1}^N \left( (\mathbf{R} \mathbf{x}_i + \mathbf{t} - \mathbf{y}_i) \cdot \mathbf{n}_i \right)^2.
  \label{eq:plane-icp-energy}
\end{equation}
Minimizing this error function requires solving a complicated constrained non-linear problem due to the orthogonality of rotation matrix $\mathbf{R} \in SO(3)$. In contrast, we can simplify the problem when we assume the amount of rotation is relatively small, namely $\theta \approx 0$. In this case, rotation matrix $\mathbf{R}(\theta, \mathbf{w})$ for the rotation around axis $\mathbf{w}$ by angle $\theta$ can be simplified to a skew symmetric matrix using the Rodrigues' rotation formula.
\begin{equation}
  \mathbf{R}(\theta, \mathbf{w}) = \mathbf{I} + \sin\theta \mathbf{K}(\mathbf{w}) + (1 - \cos\theta) \mathbf{K}^2(\mathbf{w})
  \approx \mathbf{I} + \mathbf{K}(\theta \mathbf{w})
  % = \begin{bmatrix}
  %   1           & -\theta w_z & \theta w_y  \\
  %   \theta w_z  & 1           & -\theta w_x \\
  %   -\theta w_y & \theta w_x  & 1           \\
  % \end{bmatrix}.
  \label{eq:linearized-rotation}
\end{equation}
where $\mathbf{K}(\mathbf{w}) \in \mathbb{R}^{3 \times 3}$ is a matrix corresponding to the cross product with $\mathbf{w}$ (i.e., $\mathbf{K}(\mathbf{w}) \mathbf{v} = \mathbf{w}\times\mathbf{v}$). Then, the problem to minimize \cref{eq:plane-icp-energy} can also be simplified as a linear system~\cite{chen1991object}, and we can obtain the axis-angle representation of rigid transformation by solving the system.
\begin{gather}
  \mathbf{A}
  \begin{bmatrix}
    \mathbf{a} \\
    \mathbf{t}
  \end{bmatrix} \!=\! \mathbf{b}, \;
  \mathbf{A} \!=\! \sum_{i=1}^{N}
  \begin{bmatrix}
      \mathbf{x}_i \!\times\! \mathbf{n}_i \\
      \mathbf{n}_i
  \end{bmatrix}\!
  \begin{bmatrix}
      \mathbf{x}_i \!\times\! \mathbf{n}_i \\
      \mathbf{n}_i
  \end{bmatrix}^\top, \; % \in \mathbb{R}^{6 \times 6},
  \mathbf{b} \!=\! \sum_{i=1}^N
  \begin{bmatrix}
    \mathbf{x}_i \!\times\! \mathbf{n}_i \\
    \mathbf{n}_i
  \end{bmatrix}\!
  (\mathbf{y}_i - \mathbf{x}_i) \cdot \mathbf{n}_i, \label{eq:matrix-inverse} % \in \mathbb{R}^{6}.
\end{gather}
where $\mathbf{a} = \theta\mathbf{w}$. Finally, we re-calculate a rotation matrix by the Rodrigues' formula using angle $\theta = \| \mathbf{a} \|_2$ and axis $\mathbf{w} = \mathbf{a} / \theta$ without the small-angle assumption. The rotation matrix obtained by $\theta$ and $\mathbf{w}$ is an approximation and does not minimize \cref{eq:plane-icp-energy}. Therefore, to obtain the optimal transformation $(\mathbf{R}^*, \mathbf{t}^*)$, we need to accumulate the approximated rotation matrices and translation vectors by repeating this process several times. For simplicity, we henceforth refer to this solution of accumulating small transformations as the \textit{iterative accumulation}.

%------------------------------------------------------------------------
\subsection{Gradient of error minimizing function}
\label{ssec:gradient}

When an algorithm to solve point-to-plane registration obtains the same rigid transformation from the same input consistently, we can regard $(\mathbf{R}^{*}, \mathbf{t}^{*})$ as a function of $\mathcal{X}$ and $\mathcal{Y}$. When the registration is used as a component of a neural network, the computation will be speeded-up significantly if we can calculate the analytic gradient of the function. However, the evaluation of the function involves the minimization of the error function in \cref{eq:plane-icp-energy}, and the definition of the gradient is not straightforward. Here, we leverage the implicit function theorem to derive the analytic gradient of such an energy minimizing function.

\begin{theorem}[Implicit function theorem]
  \label{thm:implicit}
  Let $g: \mathbb{R}^{n+m} \rightarrow \mathbb{R}^m$ be a real-valued function of the class $C^1$, and let $\mathbb{R}^{n+m}$ have coordinates $(\mathbf{x}, \mathbf{y})$. Fix a point $(\mathbf{x}_0, \mathbf{y}_0)$ with $g(\mathbf{x}_0, \mathbf{y}_0) = \mathbf{0}$, where $\mathbf{0} \in \mathbb{R}^m$ be the zero vector. If $\frac{\partial g(\mathbf{x}_0, \mathbf{y}_0)}{\partial \mathbf{y}}$, i.e., the Jacobian matrix of $g$ with $\mathbf{y}$ at $(\mathbf{x}_0, \mathbf{y}_0)$, is invertible, then there exist open sets $U \subset \mathbb{R}^n$ and $V \subset \mathbb{R}^m$ respectively including $\mathbf{x}_0$ and $\mathbf{y}_0$ such that there exists a unique real-valued function $h: U \rightarrow V$ of the class $C^1$ that satisfies $h(\mathbf{x}_0) = \mathbf{y}_0$ and $g(\mathbf{x}_0, h(\mathbf{x}_0)) = \mathbf{0}$ for all $\mathbf{x} \in U$. Moreover, the gradient of $h$ in $U$ are given by the matrix product $\frac{\diff h(\mathbf{x})}{\diff \mathbf{x}} = - \left( \frac{\partial g(\mathbf{x}, g(\mathbf{x}))}{\partial \mathbf{y}} \right)^{-1} \frac{\partial g(\mathbf{x}, g(\mathbf{x}))}{\partial \mathbf{x}}$.
\end{theorem}
This theorem says that we can differentiate function $h(\mathbf{x})$, which is given by solving $g(\mathbf{x}_0, \mathbf{y}_0) = \mathbf{0}$ with $\mathbf{x}_0$, when $g(\mathbf{x}, \mathbf{y})$ is differentiable near the point $(\mathbf{x}_0, \mathbf{y}_0)$. By applying this theorem to the gradient of a real-valued scalar function, which we denote as $f(\mathbf{x}, \mathbf{y})$, we obtain the following proposition.

\begin{proposition}[Gradient of function minimumizer]
  \label{prop:diffargmin}
  Let $f: \mathbb{R}^{n+m} \rightarrow \mathbb{R}$ be a real-valued function of the class $C^2$, and let $\mathbb{R}^{n+m}$ have coordinates $(\mathbf{x}, \mathbf{y})$. If the Hessian matrix $\frac{\partial^2 f(\mathbf{x},\mathbf{y})}{\partial\mathbf{y}^2} \in \mathbb{R}^{m \times m}$ is positive definite at a fixed point $(\mathbf{x}_0, \mathbf{y}_0)$, there exist open sets $U \subset \mathbb{R}^n$ and $V \subset \mathbb{R}^m$ respectively including $\mathbf{x}_0$ and $\mathbf{y}_0$ such that there exists a function $h: U \rightarrow V$ that satisfies
  \begin{equation}
    \mathbf{y}_0 = h(\mathbf{x}_0) = \argmin_{\mathbf{y} \in V} f(\mathbf{x}_0, \mathbf{y}).
    \label{eq:local-minimum}
  \end{equation}
  Moreover, the gradient of $h(\mathbf{x})$ is given as
  \begin{equation}
    \frac{\diff h(\mathbf{x})}{\diff \mathbf{x}} = -\left( \frac{\partial^2 f}{\partial \mathbf{y}^2} \right)^{-1} \frac{\partial^2 f}{\partial \mathbf{x} \partial \mathbf{y}}.
    \label{eq:diff-argmin}
  \end{equation}
\end{proposition}

\begin{proof}
  The coordinate $\mathbf{y}_0$ is a local minimum of $f(\mathbf{x}_0, \mathbf{y})$ with the fixed $\mathbf{x}_0$ due to the twice differentiability of $f(\mathbf{x}, \mathbf{y})$ and the positive definiteness of the Hessian matrix $\frac{\partial^2 f(\mathbf{x}_0, \mathbf{y}_0)}{\partial\mathbf{y}^2}$. Therefore, $h(\mathbf{x})$ can be defined as in \cref{eq:local-minimum}. On the other hand, the gradient $g(\mathbf{x}, \mathbf{y}) = \frac{\partial f(\mathbf{x}, \mathbf{y})}{\partial \mathbf{y}}$ is a function of the class $C^1$, and the Jacobian matrix $\frac{\partial g(\mathbf{x}, \mathbf{y})}{\partial \mathbf{y}}$ is invertible due to the non-singularity of the positive definite matrix. By applying \cref{thm:implicit} to $g(\mathbf{x}, \mathbf{y})$ with $g(\mathbf{x}_0, \mathbf{y}_0) = \mathbf{0}$ at the point $(\mathbf{x}_0, \mathbf{y}_0)$ of the local minimum, the gradient of $h(\mathbf{x})$ is obtained as in \cref{eq:diff-argmin}.
  \qed
\end{proof}

% For more intuitive understanding, let us consider the 3D surface defined by a function $z = f(x, y)$ of two scalars $x$ and $y$. In this case, a set of points on the surface that satisfy $g(x, y) = \frac{\partial f(x, y)}{\partial y} = 0$ forms a space curve, and its projection to the X-Y plane defines a planar curve $y = h(x) = \argmin_{y'} f(x, y')$. On the X-Y plane, the gradient direction of an implicit function $g(x, y)$ is a ratio of its derivative with $y$ to that with $x$. Hence, we obtain the same result with \cref{prop:diffargmin}.

\noindent
According to \cref{prop:diffargmin}, the gradient is independent of how the local minimum is obtained. Therefore, when the proposition is applied to the point-to-plane registration, we can calculate the analytic gradients of $\mathbf{R}^{*}$ and $\mathbf{t}^{*}$ efficiently, even when they are obtained by the iterative accumulation.

\subsection{Analytic gradients for point-to-plane registration}
\label{ssec:plane-icp}

Suppose that the point-to-point correspondences $\{ \mathbf{x}_i, (\mathbf{y}_i, \mathbf{n}_i) \}_{i=1}^N$ is already defined by a neural network, such as the one used in DCP~\cite{wang2019deep}. In this case, we assume that $\mathbf{y}_i$ and $\mathbf{n}_i$ for each $i$ are functions of the network parameters $\boldsymbol{\uptheta}$. To update parameters $\boldsymbol{\uptheta}$, we require calculating the gradients of a loss $\mathcal{L}$ with respect to $\boldsymbol{\uptheta}$, which can be developed using the chain rule.
\begin{equation}
  \frac{\partial\mathcal{L}}{\partial\boldsymbol{\uptheta}} = \sum_{i=1}^N \left(
  \frac{\partial\mathcal{L}}{\partial\mathbf{g}^*} \frac{\partial \mathbf{g}^*}{\partial\mathbf{y}_i} \frac{\partial \mathbf{y}_i}{\partial\boldsymbol{\uptheta}}
  +
  \frac{\partial\mathcal{L}}{\partial\mathbf{g}^*} \frac{\partial \mathbf{g}^*}{\partial\mathbf{n}_i} \frac{\partial \mathbf{n}_i}{\partial\boldsymbol{\uptheta}}
  \right),
  \label{eq:chain-rule}
\end{equation}
where $\mathbf{g}^{*} = (R_{00}^*, R_{01}^{*}, R_{02}^{*}, \ldots, R_{22}^{*}, t_0^{*}, t_1^{*}, t_2^{*}) \in \mathbb{R}^{12}$ is a vector for a rigid transformation. In \cref{eq:chain-rule}, the terms in the right hand side, except for $\frac{\partial\mathbf{g}^{*}}{\partial\mathbf{y}_i} \in \mathbb{R}^{12\times 3}$ and $\frac{\partial\mathbf{g}^{*}}{\partial\mathbf{n}_i} \in \mathbb{R}^{12\times 3}$, can be computed efficiently by the auto-differentiation of the neural network. Therefore, we can calculate the gradients in \cref{eq:chain-rule} efficiently if $\frac{\partial\mathbf{g}^{*}}{\partial\mathbf{y}_i}$ and $\frac{\partial\mathbf{g}^{*}}{\partial\mathbf{n}_i}$ for each $i$ have closed forms.

The important point here is that we cannot apply \cref{prop:diffargmin} directly to a constrained minimization. To relax the constraint, we introduce a quadratic penalty term $P(\mathbf{R})$ for the orthogonality of $\mathbf{R}$ and approximate $(\mathbf{R}^{*}, \mathbf{t}^*)$ as
\begin{align}
  ( \mathbf{R}^{*}, \mathbf{t}^{*} ) & \approx \argmin_{\mathbf{R} \in \mathbb{R}^{3\times 3}, \mathbf{t} \in \mathbb{R}^3} \hat{E}(\mathbf{R}, \mathbf{t}), \\
  \hat{E}(\mathbf{R}, \mathbf{t}) & = E(\mathbf{R}, \mathbf{t}) + \lambda P(\mathbf{R}),
  \quad P(\mathbf{R})  = \| \mathbf{R}^\top \mathbf{R} - \mathbf{I}_{[3]} \|_F^2, \label{eq:penalty-term}
\end{align}
where $\lambda$ is a parameter to control the strength of the penalty, $\mathbf{I}_{[d]} \in \mathbb{R}^{d \times d}$ is an identity matrix of dimension $d$, and $\| \cdot \|_F$ is the Frobenius norm of a matrix. Here, the penalty term $P(\mathbf{R})$ in \cref{eq:penalty-term}, which do not care about the sign of the determinant, is sufficient because the determinant of a rotation matrix obtained by the Rodrigues' formula is always $+1$. Therefore, \cref{prop:diffargmin} derives the gradients as
\begin{gather}
  \frac{\partial \mathbf{g}^*}{\partial \mathbf{y}_i} = \left( \frac{\partial^2 \hat{E}(\mathbf{g}^{*})}{\partial \mathbf{g}^2} \right)^{-1} \frac{\partial^2 \hat{E}(\mathbf{g}^{*})}{\partial \mathbf{y}_i \partial \mathbf{g}}, \quad
  \frac{\partial \mathbf{g}^*}{\partial \mathbf{n}_i} = \left( \frac{\partial^2 \hat{E}(\mathbf{g}^{*})}{\partial \mathbf{g}^2} \right)^{-1} \frac{\partial^2 \hat{E}(\mathbf{g}^{*})}{\partial \mathbf{n}_i \partial \mathbf{g}}. \label{eq:gradients}
\end{gather}
For simplicity, we henceforth denote $\mathbf{R}^{*}$, $\mathbf{t}^{*}$ and $\mathbf{g}^{*}$ just as $\mathbf{R}$, $\mathbf{t}$, and $\mathbf{g}$ unless otherwise noted. Here, the Hessian matrix $\frac{\partial^2 \hat{E}(\mathbf{g})}{\partial \mathbf{g}^2}$ in \cref{eq:gradients} is defined as
\begin{align}
  & \frac{\partial^2 \hat{E}(\mathbf{g})}{\partial \mathbf{g}^2} = 2 \sum_{i=1}^N (\hat{\mathbf{n}}_i \odot \hat{\mathbf{x}}_i ) (\hat{\mathbf{n}}_i \odot \hat{\mathbf{x}}_i )^\top
  + 4 \lambda \begin{bmatrix}
    \mathbf{M}                & \mathbf{0}_{[9\times 3]} \\
    \mathbf{0}_{[3 \times 9]} & \mathbf{0}_{[3\times 3]}
  \end{bmatrix} \in \mathbb{R}^{12 \times 12},
  \label{eq:ddEdrdr} \\
  & \mathbf{M} = \hat{\mathbf{R}} \odot \hat{\mathbf{R}}^\top + \mathbf{I}_{[9]} + (\mathbf{R} \mathbf{R}^\top \!-\! \mathbf{I}_{[3]}) \otimes \mathbf{I}_{[3]} + \mathbf{I}_{[3]} \otimes (\mathbf{R}^\top \mathbf{R} \!-\! \mathbf{I}_{[3]}) \in \mathbb{R}^{9 \times 9},
\end{align}
and the other derivatives in the right-hand side are as
\begin{align}
  \frac{\partial^2 \hat{E}(\mathbf{g})}{\partial \mathbf{y}_i \partial \mathbf{g}} & = -2 \left( \hat{\mathbf{n}}_i \odot \hat{\mathbf{x}}_i \right) {\hat{\mathbf{n}}_i}^\top \in \mathbb{R}^{12 \times 3}, \\
  \frac{\partial^2 \hat{E}(\mathbf{g})}{\partial \mathbf{n}_i \partial \mathbf{g}} & = 2 (\hat{\mathbf{n}}_i \odot \hat{\mathbf{x}}_i) (\mathbf{R} \mathbf{x}_i + \mathbf{t} - \mathbf{y}_i)^\top + 2 ((\mathbf{R} \mathbf{x}_i + \mathbf{t} - \mathbf{y}_i) \cdot \mathbf{n}_i) \hat{\mathbf{X}}_i \in \mathbb{R}^{12 \times 3}.
\end{align}
In these equations, we use the following notations.
\begin{align}
  & \hat{\mathbf{x}}_i = \begin{bmatrix}
    \mathbf{1}_{[3 \times 1]} \otimes \mathbf{x}_i \\
    \mathbf{1}
  \end{bmatrix} \in \mathbb{R}^{12}, \quad
  \hat{\mathbf{n}}_i = \begin{bmatrix}
    \mathbf{n}_i \otimes \mathbf{1}_{[3 \times 1]} \\
    \mathbf{n}_i
  \end{bmatrix} \in \mathbb{R}^{12}, \\
  & \hat{\mathbf{X}}_i = \begin{bmatrix}
    \mathbf{I}_{[3]} \otimes \mathbf{x}_i \\
    \mathbf{I}_{[3]}
  \end{bmatrix} \in \mathbb{R}^{12 \times 3}, \quad
  \hat{\mathbf{R}} = \begin{bmatrix}
    \mathbf{1}_{[3\times 3]} \otimes \mathbf{r}_0^\top & \mathbf{1}_{[3\times 3]} \otimes \mathbf{r}_1^\top & \mathbf{1}_{[3\times 3]} \otimes \mathbf{r}_2^\top
  \end{bmatrix} \in \mathbb{R}^{9 \times 9},
\end{align}
where $\mathbf{0}_{[m \times n]} \in \mathbb{R}^{m \times n}$ and $\mathbf{1}_{[m \times n]} \in \mathbb{R}^{m \times n}$ are the matrices whose entries are all zero and one, respectively, and $\mathbf{r}_i$ is the $i$th row vector of the rotation matrix $\mathbf{R}$. We further denote the Kronecker product of two matrices by an operator $\otimes$ and the element-wise multiplication (i.e., the Hadamard product) of two vectors or matrices by another operator $\odot$. Therefore, we can efficiently calculate the analytic gradients in \cref{eq:gradients} because these expressions consist only of basic vector and matrix algebra. Although we also need to calculate $\frac{\partial\mathcal{L}}{\partial\mathbf{x}_i}$ depending on the definitions of point pairs using a neural network, it can also be calculated efficiently, as shown in the supplementary document. %For easier understanding of these computations, we also provide a simple test code for the differentiable point-to-plane registration module in the supplementary materials.

% \begin{figure}[t]
%   \centering
%   \includegraphics[width=\linewidth]{overview.pdf}
%   \caption{Network architecture of our deep point-to-plane registration based on DCP-v2~\cite{wang2019deep}. Different from the original DCP-v2, our method receives both the positions and normals, and computes soft pointers to both of them (see \cref{ssec:imple-on-dcp}). Then, it applies analytically differentiable point-to-plane registration (see \cref{ssec:plane-icp}) instead of the Procrustes algorithm for point-to-point registration.}
%   \label{fig:overview}
% \end{figure}

\paragraph{Penalty parameter}

The remainder to define the efficient backpropagation is $\lambda$, which controls the penalty for the matrix orthogonality. Theoretically, a rotation matrix $\mathbf{R}$ obtained by the iterative accumulation in \cref{ssec:linearized-rotation} is orthogonal. Therefore, the penalty coefficient $\lambda$ should be infinity because the penalty term in \cref{eq:penalty-term} and its gradient are always zero for orthogonal $\mathbf{R}$. In contrast, the penalty term is not zero in numerical computation, and then, $\lambda$ will be finite. The gradient $\frac{\partial \hat{E}}{\partial \mathbf{R}}$ will be the zero vector at a local minimum when we assume the rotation matrix $\mathbf{R}$ is obtained by minimizing \cref{eq:penalty-term} using the penalty method. Therefore, we obtain $\lambda$ by solving $\frac{\partial \hat{E}(\lambda)}{\partial \mathbf{R}} = \mathbf{0}$ with respect to $\lambda$ by the least-squares method, which obtains
\begin{align}
  & \qquad \quad \lambda = \left| \frac{\partial P}{\partial \mathbf{g}} \cdot \frac{\partial E}{\partial \mathbf{g}} \right| \Big/ \left| \frac{\partial P}{\partial \mathbf{g}} \cdot \frac{\partial P}{\partial \mathbf{g}} \right|,\\
  \frac{\partial E}{\partial \mathbf{g}} & = \left( 2 \sum_{i=1}^{N} \left( \left( \mathbf{R} \mathbf{x}_i + \mathbf{t} - \mathbf{y}_i \right) \cdot \mathbf{n}_i \right) (\mathbf{x}'_i \odot \mathbf{n}'_i) \right)_{[0:9]} \!\in\! \mathbb{R}^{9}, \label{eq:dEdr} \\
  \frac{\partial P}{\partial \mathbf{g}} & = \text{vec} \left( 4 ({\mathbf{R}}^\top \mathbf{R} - \mathbf{I}_{[3]}) \mathbf{R} \right) \in \mathbb{R}^{9},
\end{align}
where $(\,\cdot\,)_{[0:9]}$ is the first $9$ entries of a vector, and $\text{vec}(\,\cdot\,)$ is a vector comprised of matrix entries in a row-major order. Since the penalty term is introduced only for the purpose of defining analytic gradients, we can calculate $\mathbf{R}$ and $\mathbf{t}$ by the iterative accumulation rather than solving the constrained problem.

\paragraph{Discussion}

The previous study~\cite{chen2020end} proposed a similar implicit gradient to ours for the PnP problem and calculated the gradient of a camera pose in $SE(3)$ based on the 6 DoF axis-angle representation of the rigid transformation. It also reported that other representations, such as 7 DoF quaternion with translation and 12 DoF rotation matrix with translation, did not work very well. In contrast, we proved that the gradients for a constrained minimization problem could be properly obtained with a penalty function. On the other hand, we can use a Lagrange multiplier as well to formulate an unconstrained problem rather than using the penalty term. However, the Lagrange multiplier must be optimized together with $\mathbf{R}$ and $\mathbf{t}$, and the energy function is only once differentiable with the Lagrange multiplier. Therefore, the requirement for \cref{prop:diffargmin} is not satisfied. This is the reason why we use the penalty method.

\subsection{Implementation over DCP}
\label{ssec:imple-on-dcp}

We implement our point-to-plane registration module over DCP-v2~\cite{wang2019deep}, PRNet~\cite{wang2019prnet}, and RPMNet~\cite{yew2020rpmnet} in our experiments. The details of the two latter implementations will be explained in the supplementary document due to space limitations, while the explanation below with DCP-v2 can mostly apply to them.
%\Cref{fig:overview} shows the network architecture of the extension of DCP-v2.
%In addition, we implement the module over PRNet~\cite{wang2019prnet} and RPMNet~\cite{yew2020rpmnet} as well.

We make three updates over the original DCP-v2 to introduce the point-to-plane registration module. First, our method, which is based on point-to-plane registration, uses point normals. Therefore, both positions and normals of the input points are leveraged by the feature extractor, i.e., the combination of a Siamese dynamic graph CNN (DGCNN)~\cite{wang2019dynamic} and a Transformer~\cite{vaswani2017attention}. The feature extractors output the features $\Phi_\mathcal{X} \in \mathbb{R}^{N \times C}$ and $\Phi_\mathcal{Y} \in \mathbb{R}^{M \times C}$ to $\mathcal{X}$ and $\mathcal{Y}$, respectively, where $C$ is the dimension of features.

Second, we average normal tensors rather than normal vectors to obtain \textit{soft pointers} to normals. The soft pointer to a point position is obtained by $\mathbf{c}_i = \text{softmax}(\Phi_\mathcal{Y} \Phi_{\mathbf{x}_i}^\top) \in \mathbb{R}^M$, where $\Phi_{\mathbf{x}_i}$ is the $i$th row of $\Phi_\mathcal{X}$. Then, we define a point $\mathbf{y}'_i = \sum_{j=1}^M c_{ij} \mathbf{y}_j$ corresponding to $\mathbf{x}_i$. However, a simple weighted average is unsuitable to normals because the normals estimated with input point clouds often involves the directional ambiguity and their directions can be vanished by averaging. Instead, we define a normal tensor $\mathbf{S}_j = \mathbf{n}_j \mathbf{n}_j^\top$ for each $\mathbf{y}_j$ and calculate the weighted average of tensors $\mathbf{S}'_i = \sum_{j=1}^{M} c_{ij} \mathbf{S}_j$. The normal vector $\mathbf{n}'_i$ for $\mathbf{y}'_i$ is given as an eigenvector of $\mathbf{S}'_i$ associated with its largest eigenvalue.

Finally, we apply our point-to-\textit{plane} registration module to $\mathcal{X}$ and the soft pointers $\mathcal{Y}' = \{ (\mathbf{y}'_i, \mathbf{n}'_i) \}_{i=1}^N$, as an alternative of the Procrustes algorithm for point-to-\textit{point} registration. Our point-to-plane registration module obtains $(\mathbf{R}, \mathbf{t})$ by the iterative accumulation over several iterations. Following the original DCP-v2, we train the neural network using a rigid motion loss $\mathcal{L} = \| \mathbf{R}^\top \mathbf{R}_{\rm gt} - \mathbf{I}_{[3]} \|_F^2 + \| \mathbf{t} - \mathbf{t}_{\rm gt} \|_2^2$ and Tikhonov regularization for the network weights, where $\mathbf{R}_{\rm gt}$ and $\mathbf{t}_{\rm gt}$ are the ground-truth rigid transformation. The calculated loss can be efficiently backpropagated by the proposed method.

\section{Experiments and Discussion}
\label{sec:experiments}

We implement the proposed method over DCP-v2~\cite{wang2019deep}, PRNet~\cite{wang2019prnet}, and RPMNet~\cite{yew2020rpmnet} to demonstrate the broad applicability of our point-to-plane registration module. These networks are optimized using ADAM~\cite{kingma2014adam} with the learning rate $\gamma = 0.0001$ and decay parameters $(\beta_1, \beta_2)$ = $(0.9, 0.999)$. The network is trained over 100 epochs, and the learning rate is halved at the end of every 10 epochs. The following experiments are performed using a single NVIDIA A6000 GPU.

We compare the proposed method with the previous end-to-end learning frameworks (i.e., DCP-v1 and DCP-v2~\cite{wang2019deep}, PRNet~\cite{wang2019prnet}, RPMNet~\cite{yew2020rpmnet}, DeepGMR~\cite{yuan2020deepgmr}, and PointNetLK revisited~\cite{li2021pointnetlk}). For these methods, we used the source codes provided by the authors and trained the networks by ourselves using the same datasets. Moreover, we compare our method with traditional approaches without deep learning, such as the ordinary point-to-point registration (ICP), Go-ICP~\cite{yang2013goicp}, fast global registration (FGR)~\cite{zhou2016fast}, and a pipeline comprised of feature-based ICP with 10000 RANSAC iterations followed by a refinement step based on point-to-plane distances (RANSAC-10K). We used the implementations provided by Open3D~\cite{zhou2018open3d} for ICP, FGR, and RANSAC-10K, and used the program provided by the authors for Go-ICP.

%\footnote{\url{https://github.com/yangjiaolong/Go-ICP}}

The comparison is based on the mean squared error (MSE), root mean squared error (RMSE), mean absolute error (MAE), and coefficient of determination ($\mathrm{R}^2$). Angular measurements for rotation matrices are in units of degrees. MSE, RMSE, and MAE should be zero while $\mathrm{R}^2$ should be equal to one if the rigid transformation is perfect. \Cref{fig:teaser,fig:results} show the registration results for two datasets (i.e., ModelNet40~\cite{wu2015shapenets} and ICL-NUIM~\cite{choi2015robust}) used in the following experiments, where the results are obtained by our extension over RPMNet. In tables below (i.e., \cref{tab:modelnet40,tab:modelnet40-unseen,tab:iclnuim}), the best value in each column is highlighted with a \colorbox[rgb]{1,0.8,0.85}{red box}, and the values for our extensions are highlighted by \textbf{bolded letters} when each of them is better than that of the base method. Refer to the supplementary document for further analyses with visual comparisons.

\begin{figure}[t]
  \centering
  \includegraphics[width=\linewidth]{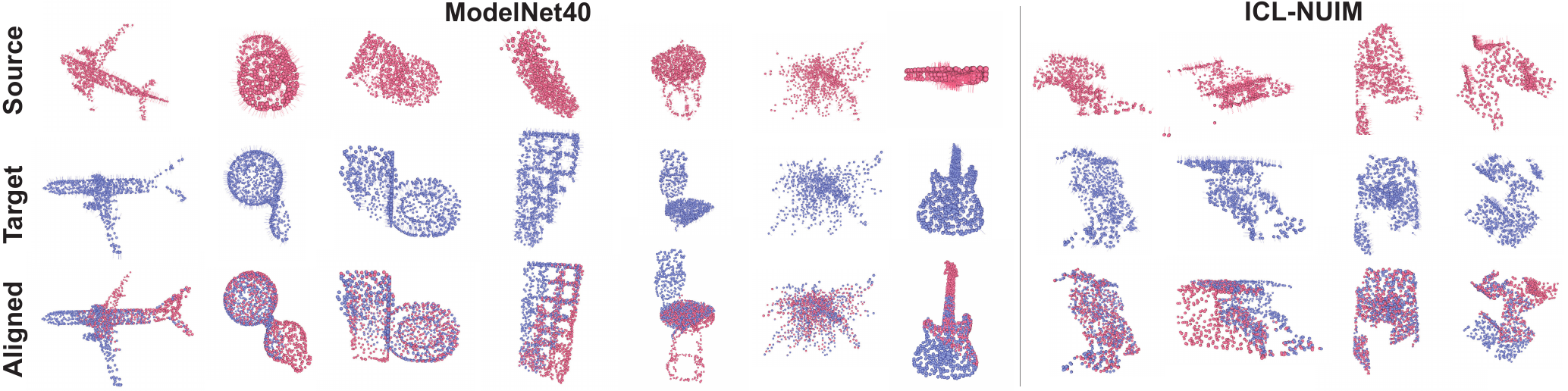}
  \caption{Point registration results obtained by our point-to-plane registration module on RPMNet~\cite{yew2020rpmnet} that performed the best on both ModelNet40 and ICL-NUIM (see \cref{tab:modelnet40,tab:modelnet40-unseen,tab:iclnuim}). For ease of recognition of shapes, we show the results for the original ModelNet40 that are obtained by the network trained with ModelNet40-CPU.}
  \label{fig:results}
\end{figure}

\begin{table*}[t]
  \centering
  \caption{Quantitative comparison on ModelNet40-CPU that includes the pairs of composed, partial, and unduplicated point clouds. As shown by the figure, our extensions over prior methods outperform the original ones, and the extension for RPMNet~\cite{yew2020rpmnet} performs the best of all the methods listed.}
  \label{tab:modelnet40}
  {\tablefontsize
    \begin{tblr}{%
      colspec={l | M[c]N[c]N[c]N[c] | N[c]N[c]N[c]N[c]},%
      hline{1,15}={0.8pt},%
      hline{2}={0.3pt},%
      hline{6,12}={0.1pt,dashed},%
      colsep=1.5pt,%
      rowsep=0.3pt,%
      width=\linewidth
    }
      Model & {{{MSE($\mathbf{R}$)$\downarrow$}}} & {{{RMSE($\mathbf{R}$)$\downarrow$}}} & {{{MAE($\mathbf{R}$)$\downarrow$}}} & {{{$\mathrm{R}^2$($\mathbf{R}$)$\uparrow$}}} & {{{MSE($\mathbf{t}$)$\downarrow$}}} & {{{RMSE($\mathbf{t}$)$\downarrow$}}} & {{{MAE($\mathbf{t}$)$\downarrow$}}} & {{{$\mathrm{R}^2$($\mathbf{t}$)$\uparrow$}}} \\
      ICP & 1580.827637 & 39.759624 & 27.912855 & -8.328819 & 0.140135 & 0.374346 & 0.286671 & -0.669430 \\
      Go-ICP~\cite{yang2013goicp} & 1153.525879 & 33.963596 & 23.638304 & -5.797102 & 0.595698 & 0.771815 & 0.589918 & -6.092995 \\
      FGR~\cite{zhou2016fast} & 165.489548 & 12.864274 & 5.914471 & 0.024253 & 0.017777 & 0.133329 & 0.085723 & 0.788066 \\
      RANSAC-10K & 114.801056 & 10.714525 & 1.665848 & 0.321655 & 0.006610 & 0.081301 & 0.017750 & 0.921368 \\
      DCP-v1~\cite{wang2019deep} & 52.947208 & 7.276483 & 5.611143 & 0.688199 & 0.011377 & 0.106665 & 0.081634 & 0.864490 \\
      DCP-v2~\cite{wang2019deep} & 47.529240 & 6.894145 & 5.133123 & 0.720141 & 0.005951 & 0.077143 & 0.058926 & 0.929107 \\
      PRNet~\cite{wang2019prnet} & 10.798091 & 3.286045 & 2.292507 & 0.936349 & 0.004012 & 0.063337 & 0.045868 & 0.952221 \\
      RPMNet~\cite{yew2020rpmnet} & 8.326686 & 2.885600 & 1.847013 & 0.950920 & 0.003124 & 0.055891 & 0.034369 & 0.962782 \\
      DeepGMR~\cite{yuan2020deepgmr} & 144.794037 & 12.033039 & 8.484650 & 0.147233 & 0.011807 & 0.108661 & 0.081751 & 0.859354 \\
      PNLK-R~\cite{li2021pointnetlk} & 539.193298 & 23.220535 & 18.872650 & -2.173193 & 0.142098 & 0.376958 & 0.280179 & -0.694167 \\
      Ours (DCP-v2) & \bold 30.991428 & \bold 5.566995 & \bold 3.814231 & \bold 0.817597 & \bold 0.004674 & \bold 0.068368 & \bold 0.049306 & \bold 0.944349 \\
      Ours (PRNet) & \bold 10.395678 & \bold 3.224233 & \bold 1.945863 & \bold 0.938705 & \bold 0.003577 & \bold 0.059811 & \bold 0.040763 & \bold 0.957411 \\
      Ours (RPMNet) & \first \bold 3.655338 & \first \bold 1.911894 & \first \bold 0.599530 & \first \bold 0.978027 & \first \bold 0.000971 & \first \bold 0.031155 & \first \bold 0.010273 & \first \bold 0.988280 \\
    \end{tblr}}
\end{table*}

\begin{table*}[t]
  \centering
  \caption{Quantitative comparison on ModelNet40-CPU-unseen. Different from \cref{tab:modelnet40-unseen}, the scores in this table are obtained by training the networks using 20 object categories, and are tested using the other 20 categories of ModelNet40.}
  \label{tab:modelnet40-unseen}
  {\tablefontsize
    \begin{tblr}{%
      colspec={l | M[c]N[c]N[c]N[c] | N[c]N[c]N[c]N[c]},%
      hline{1,15}={0.8pt},%
      hline{2}={0.3pt},%
      hline{6,12}={0.1pt,dashed},%
      colsep=1.5pt,%
      rowsep=0.3pt,%
      width=\linewidth
    }
      Model & {{{MSE($\mathbf{R}$)$\downarrow$}}} & {{{RMSE($\mathbf{R}$)$\downarrow$}}} & {{{MAE($\mathbf{R}$)$\downarrow$}}} & {{{$\mathrm{R}^2$($\mathbf{R}$)$\uparrow$}}} & {{{MSE($\mathbf{t}$)$\downarrow$}}} & {{{RMSE($\mathbf{t}$)$\downarrow$}}} & {{{MAE($\mathbf{t}$)$\downarrow$}}} & {{{$\mathrm{R}^2$($\mathbf{t}$)$\uparrow$}}} \\
      ICP & 1461.590088 & 38.230747 & 27.280455 & -7.653234 & 0.123421 & 0.351313 & 0.275050 & -0.491319 \\
      Go-ICP~\cite{yang2013goicp} & 752.880371 & 27.438665 & 12.653526 & -3.457974 & 0.098744 & 0.314236 & 0.194148 & -0.194986 \\
      FGR~\cite{zhou2016fast} & 156.695557 & 12.517810 & 5.988338 & 0.071998 & 0.018422 & 0.135729 & 0.088993 & 0.777095 \\
      RANSAC-10K & 74.711052 & 8.643556 & 1.400137 & 0.557963 & 0.005686 & 0.075405 & 0.017702 & 0.931008 \\
      DCP-v1~\cite{wang2019deep} & 69.185524 & 8.317783 & 6.399209 & 0.590056 & 0.014084 & 0.118675 & 0.092526 & 0.829668 \\
      DCP-v2~\cite{wang2019deep} & 55.167004 & 7.427449 & 5.580036 & 0.673163 & 0.006034 & 0.077676 & 0.059687 & 0.926833 \\
      PRNet~\cite{wang2019prnet} & 19.166744 & 4.377984 & 3.024774 & 0.886486 & 0.006854 & 0.082788 & 0.057999 & 0.917226 \\
      RPMNet~\cite{yew2020rpmnet} & 19.365210 & 4.400592 & 2.614891 & 0.885353 & 0.006195 & 0.078711 & 0.049118 & 0.925147 \\
      DeepGMR~\cite{yuan2020deepgmr} & 170.681854 & 13.064527 & 9.462258 & -0.010950 & 0.011070 & 0.105213 & 0.080162 & 0.865122 \\
      PNLK-R~\cite{li2021pointnetlk} & 510.894897 & 22.602983 & 16.767567 & -2.008358 & 0.159235 & 0.399043 & 0.300905 & -0.898705 \\
      Ours (DCP-v2) & \bold 38.865215 & \bold 6.234197 & \bold 4.225717 & \bold 0.769787 & \bold 0.005420 & \bold 0.073620 & \bold 0.053186 & \bold 0.934374 \\
      Ours (PRNet) & 28.377750 & 5.327077 & \bold 2.435838 & 0.832150 & \bold 0.004674 & \bold 0.068369 & \bold 0.046508 & \bold 0.943376\\
      Ours (RPMNet) & \first \bold 14.419297 & \first \bold 3.797275 & \first \bold 1.476977 & \first \bold 0.914586 & \first \bold 0.002485 & \first \bold 0.049850 & \first \bold 0.022804 & \first \bold 0.969880 \\
    \end{tblr}}
\end{table*}

\paragraph{Registration on ModelNet40}
\label{ssec:exp-modelnet40}

The ModelNet40 dataset includes 12311 synthetic CAD models from 40 object categories, of which 9843 models are for training, and the remaining 2468 models are for testing. ModelNet40 has several simple shapes, such as a simple flat plane belonging to the \textit{door} category. As reported in \cite{rusinkiewicz2001efficient,gelfand2003geometrically}, these simple shapes cause rotational and translational ambiguity in the rigid transformation that minimizes point-to-plane distances. For appropriate performance analysis of our point-to-plane registration, we perform further data augmentation on ModelNet40 to eliminate the ambiguity. Along with each point cloud sampled from either train or test split, we randomly sample two other point clouds in the same split. Then, we store points of these three models into a single point cloud $\mathcal{X}_{\rm all}$ after applying random rigid transformations. In this way, the composed point cloud $\mathcal{X}_{\rm all}$ will almost never involve the ambiguity.

After the composition, we obtain a target point cloud $\mathcal{Y}_{\rm all}$ by applying another random rigid transformation to $\mathcal{X}_{\rm all}$. Specifically, we followed DCP~\cite{wang2019deep} and PRNet~\cite{wang2019prnet} papers and used a random rigid transformation, of which the rotation along each axis is uniformly sampled from [\ang{0}, \ang{45}] and translation is in $[-0.5, 0.5]$. After the random transformation, we sample 1024 points individually from $\mathcal{X}_{\rm all}$ and $\mathcal{Y}_{\rm all}$ at random. Finally, we further sample 768 points from two point clouds to obtain $\mathcal{X}$ and $\mathcal{Y}$ by simulating the partial scan as in the PRNet paper~\cite{wang2019prnet}. After the data augmentation, $\mathcal{X}$ and $\mathcal{Y}$ correspond to different parts, and they are \textit{unduplicated}, where the points in two sets do not go to exactly the same locations after registration. It is worth noting that the previous methods, such as DCP~\cite{wang2019deep} and PRNet~\cite{wang2019prnet}, used \textit{duplicated} point clouds for evaluation, where they sampled a single point subset from the raw point cloud and transformed it with two random transformations to obtain source and target point clouds. Due to the difference, our experimental setup with \textit{unduplicated} point organizations is more challenging. Hence, the scores obtained by our experiments are slightly worse than those shown in the previous papers. We refer to the dataset after the data augmentation as ModelNet40 compose-partial-unduplicated (ModelNet40-CPU). The following experiment uses ground-truth point normals provided by ModelNet40 for FGR, RANSAC-10K, RPMNet, and our method.

As shown by the result in \cref{tab:modelnet40}, DCP-v2, PRNet, and RPMNet work well for partial and unduplicated point clouds, and by introducing the point-to-plane registration module, they get even better. Our extensions outperform the base approaches in estimating both rotation and translation. This result suggests the effectiveness of the point-to-plane registration to current deep-learning-based methods. Particularly, our method on RPMNet performs the best among all the listed approaches and is better than the original RPMNet. Although RPMNet uses point normals only to compute PPF (i.e., the part of input point features), our method performs better by leveraging the point normals also in the registration. Furthermore, almost the same results can be obtained for another test case, as shown in \cref{tab:modelnet40-unseen}, where 20 object categories are used for training, and the other 20 categories are used for evaluation. Thus, our method is robust to the shapes not included in the training data.

In addition, we found that training of DeepGMR and PNLK-R was difficult using partial and unduplicated point clouds. As for DeepGMR, we found the difficulty is due to rigorously rotationally invariant (RRI) features~\cite{chen2019clusternet}. The RRI features are not consistent for partial and unduplicated point clouds and affect the performance adversely. Even though we tested to input simple point positions instead to the network in this experiment, the network was overfitted to the training data, and the evaluation performance was not improved significantly. Similarly, we were unable to train PNLK-R with partial and unduplicated point clouds. We confirmed that PNLK-R could only be trained appropriately when the input point clouds are almost aligned but could not with the above setup where the input can rotate up to \ang{45} three times.

%\footnote{This problem is also discussed in the GitHub issue~\cite{pointnetlk2issue3}}

\begin{table}[t]
  \centering
  \caption{Quantitative comparison on the augmented ICL-NUIM dataset with simulated depth noise. For this dataset, the point normals are estimated with positions of point subsets and therefore involves random directional flips.}
  \label{tab:iclnuim}
  {\tablefontsize
    \begin{tblr}{%
      colspec={l | M[c]N[c]N[c]N[c] | N[c]N[c]N[c]N[c]},%
      hline{1,15}={0.8pt},%
      hline{2}={0.3pt},%
      hline{6,12}={0.1pt,dashed},%
      colsep=1.5pt,%
      rowsep=0.3pt,%
      width=\linewidth
    }
      Model & {{{MSE($\mathbf{R}$)$\downarrow$}}} & {{{RMSE($\mathbf{R}$)$\downarrow$}}} & {{{MAE($\mathbf{R}$)$\downarrow$}}} & {{{$\mathrm{R}^2$($\mathbf{R}$)$\uparrow$}}} & {{{MSE($\mathbf{t}$)$\downarrow$}}} & {{{RMSE($\mathbf{t}$)$\downarrow$}}} & {{{MAE($\mathbf{t}$)$\downarrow$}}} & {{{$\mathrm{R}^2$($\mathbf{t}$)$\uparrow$}}} \\
      ICP & 1239.258057 & 35.203098 & 25.900484 & -6.351683 & 0.487412 & 0.698149 & 0.383593 & -4.786241 \\
      Go-ICP~\cite{yang2013goicp} & 798.174622 & 28.251984 & 12.308638 & -3.774263 & 0.588486 & 0.767128 & 0.364457 & -6.359986 \\
      FGR~\cite{zhou2016fast} & 169.250778 & 13.009642 & 6.020503 & 0.001843 & 0.018020 & 0.134240 & 0.086465 & 0.785188 \\
      RANSAC-10K & 77.618080 & 8.810112 & 1.426499 & 0.541712 & 0.005588 & 0.074754 & 0.016481 & 0.933445 \\
      DCP-v1~\cite{wang2019deep} & 136.492889 & 11.683017 & 8.958026 & 0.160519 & 0.034493 & 0.185724 & 0.144223 & 0.579843 \\
      DCP-v2~\cite{wang2019deep} & 162.658981 & 12.753783 & 9.382932 & -0.001747 & 0.020423 & 0.142911 & 0.102363 & 0.750923 \\
      PRNet~\cite{wang2019prnet} & 37.737026 & 6.143047 & 3.812709 & 0.769193 & 0.016386 & 0.128009 & 0.085720 & 0.799919 \\
      RPMNet~\cite{yew2020rpmnet} & 2.580819 & 1.606493 & 0.929659 & 0.984234 & 0.001380 & 0.037144 & 0.020607 & 0.983145 \\
      DeepGMR~\cite{yuan2020deepgmr} & 268.978424 & 16.400562 & 12.519497 & -0.616133 & 0.375041 & 0.612406 & 0.438145 & -3.575286 \\
      PNLK-R~\cite{li2021pointnetlk} & 180.355194 & 13.429639 & 8.640598 & -0.294930 & 0.238732 & 0.488602 & 0.309280 & -2.454140 \\
      Ours (DCP-v2) & \bold 76.291809 & \bold 8.734518 & \bold 5.678668 & \bold 0.534442 & 0.022302 & 0.149339 & \bold 0.097397 & 0.728618 \\
      Ours (PRNet) & 39.333748 & 6.271662 & \bold 3.285756 & 0.762722 & \bold 0.011383 & \bold 0.106691 & \bold 0.067011 & \bold 0.861624 \\
      Ours (RPMNet) & \first \bold 0.940908 & \first \bold 0.970004 & \first \bold 0.481255 & \first \bold 0.994201 & \first \bold 0.000682 & \first \bold 0.026109 & \first \bold 0.012381 & \first \bold 0.991713 \\
    \end{tblr}}
\end{table}

\paragraph{Registration on ICL-NUIM}

Next, we perform an experiment with the augmented ICL-NUIM dataset~\cite{choi2015robust} based on simulated scans of synthetic indoor scenes. We start with the preprocessed data provided by the authors of DeepGMR~\cite{yuan2020deepgmr}, which consists of 1278 training samples and 200 evaluation samples. Points in this dataset, which are synthesized from depth images with known camera poses, do not have normals. Therefore, we calculate the point normals using a function provided by Open3D~\cite{zhou2018open3d}. After that, we process the point clouds as performed for ModelNet40-CPU to validate the effectiveness of the partial and unduplicated point clouds. However, we did not perform the composition because the sufficient geometric complexity of point clouds in ICL-NUIM could avoid the ambiguity of rigid transformation without compositing.

\Cref{tab:iclnuim} shows the results of comparison for ICL-NUIM, based on the same metrics used in \cref{tab:modelnet40}. Even with ICL-NUIM, which is more similar to real scans, the tendency of evaluation scores is approximately the same as those for ModelNet40-CPU based on artificial CAD models. Our extensions mostly outperform the base methods, and the extension over RPMNet~\cite{yew2020rpmnet} performs the best. The ICL-NUIM dataset includes simulated sensor noise, and the point normals are estimated with such noisy point positions. Therefore, the good performance of our method for ICL-NUIM demonstrates the robustness of our method for realistic scans with noisy point positions and estimated normals.

% \begin{figure}[tbp]
% \centering
% \includegraphics[width=\linewidth]{numeric_errors.pdf}
% \caption{Rigid motion loss and numeric errors for its gradients are compared between our analytic differentiation and auto-differentiation. Errors are sufficiently small and almost converged when ten ICP iterations are performed.}
% \label{fig:numeric-errors}
% \end{figure}

\paragraph{Computational complexity}

Compared to the point-to-point registration, which can be solved using SVD, the point-to-plane registration spends more time and memory due to the iterative accumulation. Specifically, the forward pass requires more computation for the iteration, and the backward pass requires more computation for matrix inversions in \cref{eq:gradients}. However, the dominant computation both in training and testing is the evaluation of the neural network. As shown in \cref{tab:performance-vs-svd}, the additional memory that is required to install our method is almost negligible, and the additional time will be compromised considering the improvement of their performance.

\begin{table}[t]
  \centering
  \caption{The amount of computation time and memory that are additionally required when we replace the SVD-based point-to-point registration module in the base methods with our point-to-plane registration module.}
  \label{tab:performance-vs-svd}
  {\fontsize{6pt}{7pt}\selectfont
  \begin{tblr}{%
    colspec={X[l]X[c]X[c]X[c]X[c]X[c]X[c]},%
    hline{1,5}={0.8pt},%
    hline{2}={2-7}{0.3pt},%
    hline{3}={0.3pt},%
    vline{2,4,6}={0.3pt},%
    vline{3,5,7}={0.3pt,dotted},%
    colsep=1.5pt,%
    rowsep=0.2pt,%
  }
   & \SetCell[c=2]{c} DCP-v2~\cite{wang2019deep} & & \SetCell[c=2]{c} PRNet~\cite{wang2019prnet} & & \SetCell[c=2]{c} RPMNet~\cite{yew2020rpmnet} & \\
   & train & test & train & test & train & test \\
  Time & \qty{17.5}{\percent} & \qty{23.3}{\percent} & \qty{27.7}{\percent} & \qty{46.2}{\percent} & \qty{15.5}{\percent} & \qty{20.4}{\percent} \\
  Memory & \qty{0.50}{\percent} & \qty{0.27}{\percent} & \qty{1.15}{\percent} & \qty{0.86}{\percent} & \qty{0.40}{\percent} & \qty{0.38}{\percent} \\
  \end{tblr}}
\end{table}

\begin{figure}[t]
  \centering
  \includegraphics[width=\linewidth]{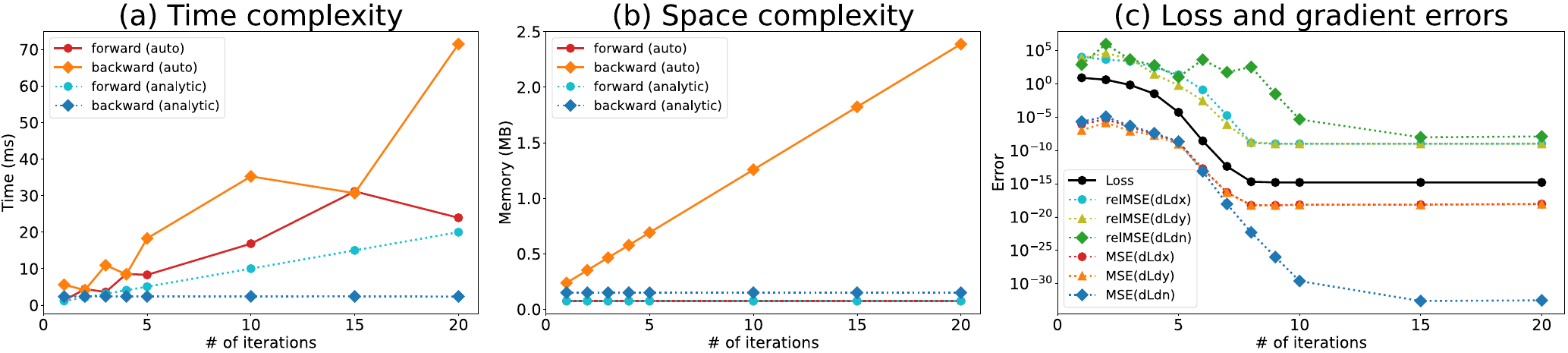}
  \caption{\textbf{\textsf{(a), (b)}} Comparison of time and space complexity using a point cloud with 1024 points.
  % Backpropagation by our point-to-plane registration with analytic differentiation is 14.8 times and 8.4 times more efficient in time and in memory than that by auto-differentiation, respectively.
  \textbf{\textsf{(c)}} The rigid motion loss and numeric errors for its gradients are compared between our analytic differentiation and auto-differentiation.
  %Errors are sufficiently small and almost converged when linearized rotation matrices are accumulated ten times.}
  }
  \label{fig:performance}
\end{figure}

\paragraph{Analytic differentiation vs auto-differentiation}

We compare the time and space complexity of our analytic differentiation on point-to-plane registration with simple auto-differentiation using a large computation graph given as a result of the iterative accumulation. The experiment uses a single point cloud with 1024 points and measures the time and graphics memory consumed by either analytic or automatic differentiation. Figures~\ref{fig:performance}(a) and \ref{fig:performance}(b) show the results. In the forward pass, the time performance of the analytic differentiation is higher than that of the auto-differentiation because the auto-differentiation involves slight computational overhead to prepare the computation graph for backpropagation, whereas the graphics memory used by each method is exactly the same. In contrast, the time and memory consumption in the backward pass differ. The analytic differentiation reduces required time and memory significantly compared to the auto-differentiation. For example, when the iterative accumulation is repeated ten times, the analytic differentiation is 14.8$\times$ and 8.4$\times$ more efficient in time and in memory than the auto-differentiation, respectively. Especially for training, the time overhead with the auto-differentiation is estimated as $\qty{15.5}{\%} \times 14.8 = \qty{229.4}{\%}$. Thus, the analytic differentiation is essential to efficient training. 

\paragraph{Number of iterations and numeric errors}

Our analytic differentiation requires to get \cref{eq:plane-icp-energy} converge to either local or global minimum because \cref{prop:diffargmin} says $\frac{\partial E(\mathbf{g})}{\partial \mathbf{g}} = \mathbf{0}$ is a necessary condition to obtain \cref{eq:diff-argmin}. However, it is not always satisfied in practice and causes numeric errors, depending on the number of iterations of accumulating small-angle rigid transformations. Therefore, we evaluate the numeric errors of the gradients (i.e., $\frac{\partial\mathcal{L}}{\partial\mathbf{x}_i}$, $\frac{\partial\mathcal{L}}{\partial\mathbf{y}_i}$, and $\frac{\partial\mathcal{L}}{\partial\mathbf{n}_i}$) in mean square error (MSE) and relative MSE (relMSE) by comparing those obtained by the auto-differentiation and our analytic differentiation. As shown by \cref{fig:performance}(c), the numeric errors decrease with the number of iterations, and when we repeat the accumulation ten times, they are sufficiently small and approximately converged. As shown by all these experiments, ten iterations are sufficient to obtain better performances than the base methods.

\paragraph{Point-to-plane registration for only testing}

We evaluate the necessity of training the neural network with the point-to-\textit{plane} registration module because it could be enough if that is used only in evaluation after the network is trained with the point-to-\textit{point} registration module. For each of DCP-v2~\cite{wang2019deep}, PRNet~\cite{wang2019prnet}, and RPMNet~\cite{yew2020rpmnet}, we use the network weights trained by ModelNet40-CPU and tested them by replacing their point-to-\textit{point} registration module with the point-to-\textit{plane} registration module. As shown by \cref{tab:plane-test}, the performances are degraded compared to those of our extensions in \cref{tab:modelnet40} and are further worse than those trained and tested consistently with the point-to-point registration module. Therefore, training with the point-to-plane registration module is indispensable for the good performance of our method.

\begin{table}[t]
  \centering
  \caption{Quantitative evaluation on ModelNet40-CPU for the networks trained with the point-to-\textit{point} registration but tested with the point-to-\textit{plane} registration. This table shows the necessity of training using the point-to-plane registration module.}
  \label{tab:plane-test}
  {\tiny
    \begin{tblr}{%
      colspec={l | M[c]N[c]N[c]N[c] | N[c]N[c]N[c]N[c]},%
      colsep=3pt,%
      rowsep=0.3pt,%
      hline{1,5}={0.8pt},
      hline{2}={0.3pt},
      width=\linewidth
    }
      Model & {{{MSE($\mathbf{R}$)$\downarrow$}}} & {{{RMSE($\mathbf{R}$)$\downarrow$}}} & {{{MAE($\mathbf{R}$)$\downarrow$}}} & {{{$\mathrm{R}^2$($\mathbf{R}$)$\uparrow$}}} & {{{MSE($\mathbf{t}$)$\downarrow$}}} & {{{RMSE($\mathbf{t}$)$\downarrow$}}} & {{{MAE($\mathbf{t}$)$\downarrow$}}} & {{{$\mathrm{R}^2$($\mathbf{t}$)$\uparrow$}}} \\
      DCP-v2~\cite{wang2019deep} & 154.640244 & 12.435443 & 8.170769 & 0.089408 & 0.045395 & 0.213062 & 0.146939 & 0.459433 \\
      PRNet~\cite{wang2019prnet} & 33.430279 & 5.781892 & 3.431526 & 0.803085 & 0.008348 & 0.091366 & 0.066771 & 0.900555 \\
      RPMNet~\cite{yew2020rpmnet} & 20.276587 & 4.502953 & 1.751912 & 0.880535 & 0.003380 & 0.058135 & 0.026753 & 0.959721 \\
    \end{tblr}}
\end{table}

\paragraph{Limitation}

The first limitation is that our method requires point normals. However, our method works well with point normals estimated from point positions, as demonstrated with ICL-NUIM. Furthermore, we may obtain more performance using higher-quality point normals estimated by recent deep-learning-based methods~\cite{boulch2016deep,shabat2019nestinet}. The second limitation is that our method does not work well with simple shapes, such as a flat plane, because of the ambiguity of transformation that we discussed previously. Nevertheless, this would not be a problem in practice because real-world point clouds are significantly more complicated.

\section{Conclusion}
\label{sec:conclusion}

In this paper, we presented an efficient backpropagation method for rigid transformations obtained via error minimization, achieving a new application of neural networks to point-to-plane registration. As we demonstrated by the experiments, our point-to-plane registration module has improved the performances of the previous methods based on point-to-point registration. Although our efficient backpropagation has been tested with a point-to-plane registration based on the iterative accumulation of small rotations and translations, the basic idea will be useful for registration algorithms. For future work, we would like to investigate the applications to the algorithms minimizing other distances such as plane-to-plane distances~\cite{segal2009generalized} and symmetric distances~\cite{rusinkiewicz2019symmetric}.

\clearpage
\bibliographystyle{splncs}
\bibliography{references}

\clearpage

\fi
\ifsupl

% ------------------------------------------------------------------------------------------------------

\title{Deep Point-to-Plane Registration by Efficient Backpropagation for Error Minimizing~Function}
\ifdraft
  \titlerunning{ACCV-22 submission ID \ACCVSubNumber}
  \authorrunning{ACCV-22 submission ID \ACCVSubNumber}
  \author{Anonymous ACCV submission}
  \institute{Paper ID \ACCVSubNumber}
\else
  \titlerunning{Deep Point-to-Plane Registration}
  \author{Tatsuya Yatagawa%\orcidID{0000-0003-4653-2435}
    \and Yutaka Ohtake \and Hiromasa Suzuki}
  \authorrunning{T. Yatagawa et al.}
  \institute{School of Engineering, The University of Tokyo, Tokyo, Japan
    \email{\{tatsy,ohtake,suzuki\}@den.t.u-tokyo.ac.jp}}
\fi

% Add "supplementary document" to title
\makeatletter
\let\maintitle\@title
\renewcommand{\@title}{\maintitle\texorpdfstring{\\}{}---Supplementary Document---}
\makeatother

% Numbering and section titling
\setcounter{page}{1}
\setcounter{section}{0}
\setcounter{equation}{0}
\setcounter{figure}{0}
\setcounter{table}{0}
\renewcommand{\thesection}{\Alph{section}}
\renewcommand{\theequation}{\Alph{section}\arabic{equation}}
\renewcommand{\thefigure}{\Alph{section}\arabic{figure}}
\renewcommand{\thetable}{\Alph{section}\arabic{table}}
\counterwithin*{equation}{section}
\counterwithin*{figure}{section}
\counterwithin*{table}{section}

\backrefsetup{disable}

\ifdraft
  \cvprrulercount=0
\fi

\crefname{section}{Appendix}{Appendix}

% Change "Require" and "Ensure" to "Input" and "Output", respectively
\renewcommand{\algorithmicrequire}{\textbf{Input:}}
\renewcommand{\algorithmicensure}{\textbf{Output:}}
\newcommand{\hyphen}{\text{-}}

% Show title
\maketitle

\section{Implementation on PRNet}
\label{sec:plane-prnet}

For partial-to-partial point set registration, DCP~\cite{wang2019deep} is extended by PRNet~\cite{wang2019prnet} that employs actor-critic closest point network (ACNet) to control the softness of pointers between source and target point clouds (i.e., $\mathcal{X}$ and $\mathcal{Y}$, respectively). \Cref{algo:prnet} shows an outline of the algorithm that extends PRNet using our point-to-plane registration module. When we implement the point-to-plane registration module over PRNet, both the positions and normals of points are supplied to the feature extractor, i.e., Siamese DGCNN~\cite{wang2019dynamic} and Transformer~\cite{vaswani2017attention}, as the extension for DCP-v2. By denoting the feature extractor as $F$, the features for $\mathcal{X}$ and $\mathcal{Y}$ are given by
\begin{equation}
  \Phi_{\mathcal{X}} = F(\mathcal{X}), \quad
  \Phi_{\mathcal{Y}} = F(\mathcal{Y}).
\end{equation}
Since the number of points in $\mathcal{X}$ and $\mathcal{Y}$ are not always the same in partial-to-partial registration, PRNet employs simple keypoint extraction, where $K$ points are chosen as keypoints from each point cloud in ascending order of the $L_2$ norm of a feature vector. We henceforth denote the positions and normals of these keypoints with \textit{hat} signs, such as $\hat{\mathcal{X}}$ and $\hat{\mathcal{Y}}$. Then, the softness $\tau$ is defined by another neural network, referred to as the temperature network. By denoting the temperature network as $T$, we obtain
\begin{equation}
  \tau = T(\Phi_{\hat{\mathcal{X}}}, \Phi_{\hat{\mathcal{Y}}}).
\end{equation}
PRNet defines hard point-to-point correspondences using Gumbel--Softmax, rather than the ordinary softmax operation used by DCP. The weight parameters $\hat{\mathbf{c}}_i$, which are used to define the pointers from $\hat{\mathbf{x}}_i$ to $\hat{\mathcal{Y}}$, is defined as
\begin{equation}
  \hat{\mathbf{c}}_i = \text{one-hot} \left[ \argmax_j \text{softmax} \left( \frac{\Phi_{\hat{\mathcal{Y}}} \Phi_{\hat{\mathbf{x}}_i}^\top + \mathbf{q}_{i}}{\tau} \right) \right],
  \label{eq:gumbel-softmax}
\end{equation}
where $q_{i1}, q_{i2}, \ldots, q_{iN}$ (i.e., the entries of $\mathbf{q}_i$) are independently and identically distributed (i.i.d.) samples drawn from the standard Gumbel distribution (i.e., $\mathrm{Gumbel}(0, 1)$). The pointers from each $\hat{\mathbf{x}}_i$ to $\hat{\mathcal{Y}}$ is defined equivalently with that in the extension of DCP-v2, regardless of the hardness of $\hat{\mathbf{c}}_i$. The pointer to a position $\hat{\mathbf{y}}_i$ is defined by the weighted sum of positions in $\hat{\mathcal{Y}}$, whereas that to a normal $\hat{\mathbf{n}}_i$ is defined using the weighted sum of normal tensors.

\begin{algorithm}[t]
  \caption{Point-to-plane PRNet}
  \label{algo:prnet}
  \begin{algorithmic}[1]
    \Require{$\mathcal{X} = \{ (\mathbf{x}_i, \mathbf{m}_i) \}_{i=1}^N$, $\mathcal{Y} = \{ (\mathbf{y}_j, \mathbf{n}_j) \}_{j=1}^M$, and $N_{ac\hyphen iter}$.}
    \Ensure{$\mathbf{R}^{*}$ and $\mathbf{t}^{*}$.}
    \State $\mathbf{R}^{*} = \mathbf{I}_{[3]}$, $\mathbf{t}^{*} = \mathbf{0}_{[3\times 1]}$.
    \For{$i=1, \ldots, N_{ac\hyphen iter}$}
    \State $\Phi_{\mathcal{X}} = F(\mathcal{X})$
    \State $\Phi_{\mathcal{Y}} = F(\mathcal{Y})$
    \State $\hat{\mathcal{X}} = \mathcal{X}(\text{topk}(\| \Phi_{\mathbf{x}_1} \|_2, \ldots, \| \Phi_{\mathbf{x}_N} \|_2))$
    \State $\hat{\mathcal{Y}} = \mathcal{Y}(\text{topk}(\| \Phi_{\mathbf{y}_1} \|_2, \ldots, \| \Phi_{\mathbf{y}_M} \|_2))$
    \State $\tau = T(\Phi_{\hat{\mathcal{X}}}, \Phi_{\hat{\mathcal{Y}}})$
    \State Calculate $\hat{\mathbf{c}}_i$ by \cref{eq:gumbel-softmax} and soft pointers from $\hat{\mathcal{X}}$ to $\hat{\mathcal{Y}}$
    \State Compute $\mathbf{R}_i, \mathbf{t}_i$ by point-to-plane registration (i.e., iterative accumulation)
    \State $\mathcal{X} \leftarrow \{ (\mathbf{R}_i \mathbf{x} + \mathbf{t}_i, \mathbf{R}_i \mathbf{m}_i) : \mathbf{x} \in \mathcal{X} \}$
    \State $\mathbf{R}^{*} \leftarrow \mathbf{R}_i \mathbf{R}^{*}$
    \State $\mathbf{t}^{*} \leftarrow \mathbf{R}_i \mathbf{t}^{*} + \mathbf{t}_i$
    \EndFor
  \end{algorithmic}
\end{algorithm}

Finally, we apply our point-to-plane registration module to the pairs of corresponding keypoints defined by the above procedure. In addition to the iterative accumulation for point-to-plane registration, we perform several actor-critic iterations. To obtain the results shown in the main text, we performed two actor-critic iterations for training and six iterations for testing, while the iterative accumulation is repeated ten times in each actor-critic iteration. The soft pointers are updated after each actor-critic iteration by applying the rigid transformation obtained by the registration module to the source points $\mathcal{X}$.

As we mentioned in the main text, we assume the rigid transformation obtained by each actor-critic iteration is a function of network parameters. Since the positions and normals of $\mathcal{X}$ is updated after the actor-critic iteration, those used in the second and later iterations are dependent on the network parameters. Therefore, we need the gradients of our point-to-plane registration module with respect also to $\mathbf{x}_i$. \Cref{prop:diffargmin} can derive the analytic gradient of rigid transformation $\mathbf{g}$ with respect also to $\mathbf{x}_i$ as
\begin{align}
  \frac{\partial \mathbf{g}^{*}}{\partial\mathbf{x}_i} & = -\left( \frac{\partial^2 \hat{E}(\mathbf{g}^{*})}{\partial \mathbf{g}^2} \right)^{-1} \frac{\partial^2 \hat{E}(\mathbf{g}^{*})}{\partial\mathbf{x}_i \partial\mathbf{g}} \in \mathbb{R}^{12 \times 3}, \\
  \frac{\partial^2 \hat{E}(\mathbf{g}^{*})}{\partial\mathbf{x}_i \partial\mathbf{g}}
  & = 2 (\hat{\mathbf{n}}_i \odot \hat{\mathbf{x}}_i) (\mathbf{R}^\top \mathbf{n}_i )^\top + 2 \begin{bmatrix}
    (((\mathbf{R}\mathbf{x}_i + \mathbf{t} - \mathbf{y}_i) \cdot \mathbf{n}_i) \mathbf{n}_i ) \otimes \mathbf{I}_{[3]} \\
    \mathbf{0}_{[3 \times 3]}
  \end{bmatrix} \in \mathbb{R}^{12 \times 3}.
\end{align}
Using this formula, as well as \cref{eq:gradients}, we can efficiently backpropagate the loss through the point-to-plane registration module, even when used with PRNet.

\section{Implementation on RPMNet}
\label{sec:plane-rpmnet}

RPMNet~\cite{yew2020rpmnet} is another approach to partial-to-partial registration that applies a deep neural network to the robust point matching~\cite{gold1998rpm}. In addition to the soft pointers between $\mathcal{X}$ and $\mathcal{Y}$, RPMNet calculates the reliability weights for the correspondences. We obtain the best rigid transformation by the weighted SVD using the reliability weights. To extend RPMNet using the point-to-plane registration module, we update the definition of the match matrix, which is used to define point correspondences, and replace the weighted SVD with the registration module that also uses weights. Although we modified the feature extractors of DCP-v2 and PRNet to receive point normals, we did not modify that of RPMNet for the point-to-plane extension because RPMNet has already used point normals to compute PPFs, which is a part of input features.

\subsection{Hard match matrix}

RPMNet defines soft correspondences between point clouds $\mathcal{X}$ and $\mathcal{Y}$ by a match matrix $\mathbf{Z} \in \mathbb{R}^{N \times M}$, where $N$ and $M$ are the numbers of points in $\mathcal{X}$ and $\mathcal{Y}$, respectively. A matrix entry $z_{ij}$ of $\mathbf{Z}$ represents the significance of soft assignment from $\mathbf{y}_j$ to $\mathbf{x}_i$, which is defined as
\begin{equation}
  z_{ij} = \exp (-\beta \| \mathbf{R} \mathbf{x}_i + \mathbf{t} - \mathbf{y}_j \|_2^2 + \alpha),
  \label{eq:match-matrix}
\end{equation}
where $\alpha$ is a parameter to control the number of correspondences rejected as outliers, and $\beta$ is the annealing parameter to control the hardness of correspondences. A point $\mathbf{y}'_i$, which is a soft pointer from $\mathbf{x}_i$, is then defined as
\begin{equation}
  \mathbf{y}'_i = \frac{1}{\sum_{j=1}^M z_{ij}} \sum_{j=1}^M z_{ij} \mathbf{y}_j
  \label{eq:rpm-weight-average}
\end{equation}
Straightforwardly, we can also define a soft pointer to normals by weighting normal tensors analogously to \cref{eq:rpm-weight-average}, as we explained in the main text. However, we found that the values of ${z}_{ij}$ were almost uniform when the network to obtain $\alpha$ and $\beta$ is not trained sufficiently, which can diminish the anisotropy of normal vectors and cause ambiguity in the best rigid transformation.

To alleviate this problem, we reformulate the weight average in \cref{eq:rpm-weight-average}. Let us define a vector $\mathbf{u}_i$, where $u_{ij} = -\beta \| \mathbf{R} \mathbf{x}_i + \mathbf{t} - \mathbf{y}_j \|_2^2 + \alpha$. Then, we can rewrite \cref{eq:rpm-weight-average} using the softmax operation:
\begin{equation}
  \mathbf{y}'_i = \sum_{j=1}^{M} c_{ij} \mathbf{y}_j, \quad \mathbf{c}_{i} = \text{softmax}(\mathbf{u}_i).
\end{equation}
Although this equation means exactly the same as \cref{eq:rpm-weight-average}, we modify the weights $\mathbf{c}_{i}$ using Gumbel--Softmax:
\begin{equation}
  \mathbf{c}_i = \text{one-hot} \left[ \argmax_j \text{softmax} \left( \frac{\mathbf{u}_i + \mathbf{q}_i}{\tau} \right) \right],
\end{equation}
where $\mathbf{q}_i$ is a vector with i.i.d. samples from $\text{Gumbel}(0, 1)$, and $\tau$ is a parameter to control the softness. In RPMNet, we use the constant $\tau = 1$ because the softness has already been controlled by $\beta$. The hard correspondences defined by Gumbel--Softmax prevents normals from being averaged to the same direction and allows the point-to-plane registration module to calculate rigid transformation appropriately.

\subsection{Weighted point-to-plane registration}

Following the original RPMNet, we define the reliability weight for $i$th  point pair as $\zeta_i = \sum_{j=1}^M z_{ij}$. Using this reliability, we redefine the point-to-plane error function in \cref{eq:plane-icp-energy} as
\begin{equation}
  E(\mathbf{R}, \mathbf{t}) = \sum_{i=1}^N \zeta_i \left( (\mathbf{R} \mathbf{x}_i + \mathbf{t} - \mathbf{y}_i) \cdot \mathbf{n}_i \right)^2.
  \label{eq:weighted-plane-icp-error}
\end{equation}
We can minimize this error function by linearizing the rotation matrix as in the main text. The linear system in \cref{eq:matrix-inverse} is updated with
\begin{align}
  \mathbf{A} & = \sum_{i=1}^{N} \zeta_i \left(
  \begin{bmatrix}
      \mathbf{x}_i \times \mathbf{n}_i \\
      \mathbf{n}_i
    \end{bmatrix}
  \begin{bmatrix}
      \mathbf{x}_i \times \mathbf{n}_i \\
      \mathbf{n}_i
    \end{bmatrix}^\top
  \right) \in \mathbb{R}^{6 \times 6}, \\
  \mathbf{b} & = \sum_{i=1}^N \zeta_i \left(
  \begin{bmatrix}
    \mathbf{x}_i \times \mathbf{n}_i \\
    \mathbf{n}_i
  \end{bmatrix}
  (\mathbf{y}_i - \mathbf{x}_i) \cdot \mathbf{n}_i \right) \in \mathbb{R}^{6}.
\end{align}
Finally, we obtain the rigid transformation by solving this linear system.

For the extension above of RPMNet, we need to calculate the analytic gradient of the rigid transformation, which is obtained by the point-to-plane registration, with respect also to $\zeta_i$ because the weights are obtained by a neural network. \Cref{prop:diffargmin} derives the analytic gradient of rigid transformation $\mathbf{r}$ with respect to $\zeta_i$ as
\begin{align}
  \frac{\partial\mathbf{g}^{*}}{\partial \zeta_i} & = - \left( \frac{\partial^2 \hat{E}(\mathbf{g}^{*})}{\partial\mathbf{g}^2} \right)^{-1} \frac{\partial^2 \hat{E}(\mathbf{g}^{*})}{\partial \zeta_i \partial \mathbf{g}}, \\
  \frac{\partial^2 \hat{E}(\mathbf{g}^{*})}{\partial \zeta_i \partial\mathbf{g}} & = 2 (\hat{\mathbf{n}}_i \odot \hat{\mathbf{x}}_i) \left( (\mathbf{R} \mathbf{x}_i + \mathbf{t} - \mathbf{y}_i) \cdot \mathbf{n}_i \right) \in \mathbb{R}^{12 \times 1}
\end{align}
Based on the updated error function in \cref{eq:weighted-plane-icp-error}, we also update \cref{eq:ddEdrdr} as
\begin{align}
  \frac{\partial^2 \hat{E}(\mathbf{g}^{*})}{\partial \mathbf{g}^2} & = 2 \sum_{i=1}^N \zeta_i (\hat{\mathbf{n}}_i \odot \hat{\mathbf{x}}_i ) (\hat{\mathbf{n}}_i \odot \hat{\mathbf{x}}_i )^\top + 4 \lambda \begin{bmatrix}
    \mathbf{M}                & \mathbf{0}_{[9\times 3]} \\
    \mathbf{0}_{[3 \times 9]} & \mathbf{0}_{[3\times 3]}
  \end{bmatrix} \in \mathbb{R}^{12 \times 12}.
\end{align}
In addition, we update $\frac{\partial^2 \hat{E}}{\partial \mathbf{x}_i \partial\mathbf{g}}$, $\frac{\partial^2 \hat{E}}{\partial \mathbf{y}_i \partial\mathbf{g}}$, and $\frac{\partial^2 \hat{E}}{\partial \mathbf{n}_i \partial\mathbf{g}}$ by multiplying $\zeta_i$ to each of them. In contrast, we update the penalty parameter $\lambda$ by multiplying $\zeta_i$ to each summand in \cref{eq:dEdr}.

\section{Additional Experiments}
\label{apdx:additional-experiments}

Here, we conduct further experiments to elaborate on the differences between the base methods (i.e., DCP-v2, PRNet, and RPMNet) by point-to-point registration and their extensions by point-to-plane registration.

First, we show the cumulative distribution functions (CDFs) of three metrics (i.e., RMSE of rotation angles ($\mathbf{R}$), RMSE of translation vectors ($\mathbf{t}$), and modified chamfer distances~\cite{yew2020rpmnet}) in \cref{fig:modelnet-cdf,fig:iclnuim-cdf}. These two figures show the charts for ModelNet40-CPU and ICL-NUIM, respectively. In each chart, a point $(x, y)$ on the curve indicates that the values based on the metric are less than $x$ for $100y$\,\% of models. Therefore, the curve of a method drawn higher than that of another one means better performance. As shown in these charts, the curves of our extensions are located above the corresponding base methods when the values in \cref{tab:modelnet40,tab:iclnuim} are better than those of the base methods. Compared to the single scalar values in \cref{tab:modelnet40,tab:iclnuim}, these results are more informative and also support the advantage of our method over the base methods by point-to-point registration.

\begin{figure}[t]
  \centering
  \includegraphics[width=\linewidth]{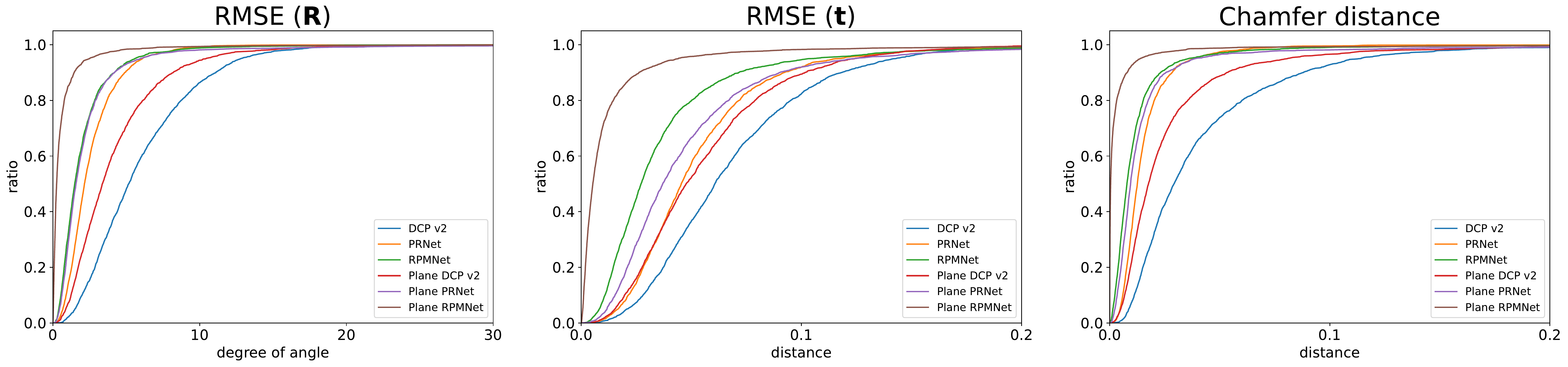}
  \caption{CDF of three metrics (i.e., RMSE ($\mathbf{R}$), RMSE ($\mathbf{t}$), and modified chamfer distance~\cite{yew2020rpmnet}) on the test set of the ModelNet40-CPU dataset. In these charts, the curve of a method drawn higher than that of another means the better performance.}
  \label{fig:modelnet-cdf}
\end{figure}

\begin{figure}[t]
  \centering
  \includegraphics[width=\linewidth]{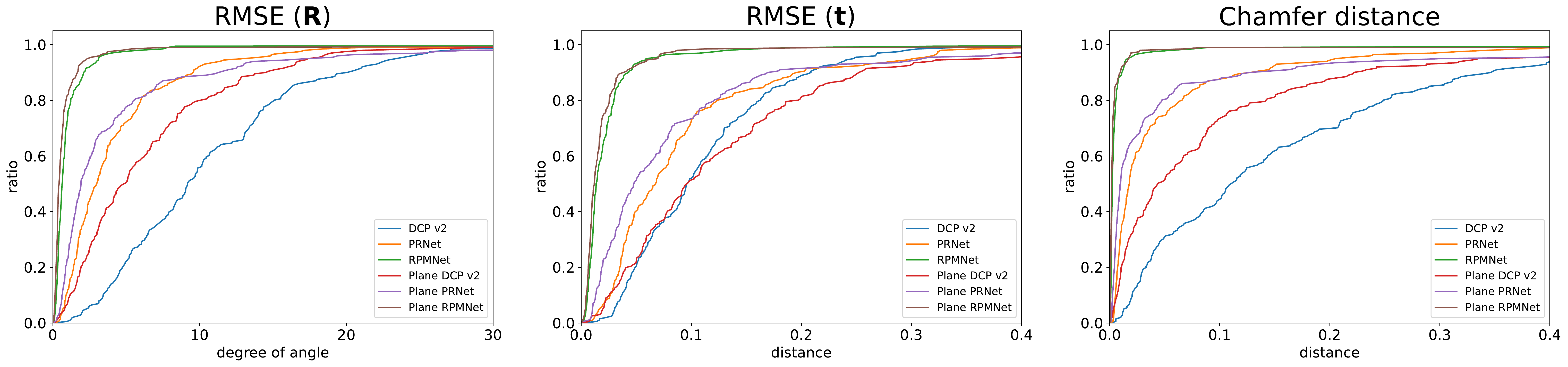}
  \caption{CDF of three metrics (i.e., RMSE ($\mathbf{R}$), RMSE ($\mathbf{t}$), and modified chamfer distance~\cite{yew2020rpmnet}) on the test set of the ICL-NUIM dataset. In these charts, the curve of a method drawn higher than that of another means the better performance.}
  \label{fig:iclnuim-cdf}
\end{figure}

\begin{figure}[t]
  \centering
  \begin{tblr}{colspec={X[c]X[c]}}
    \includegraphics[width=\linewidth]{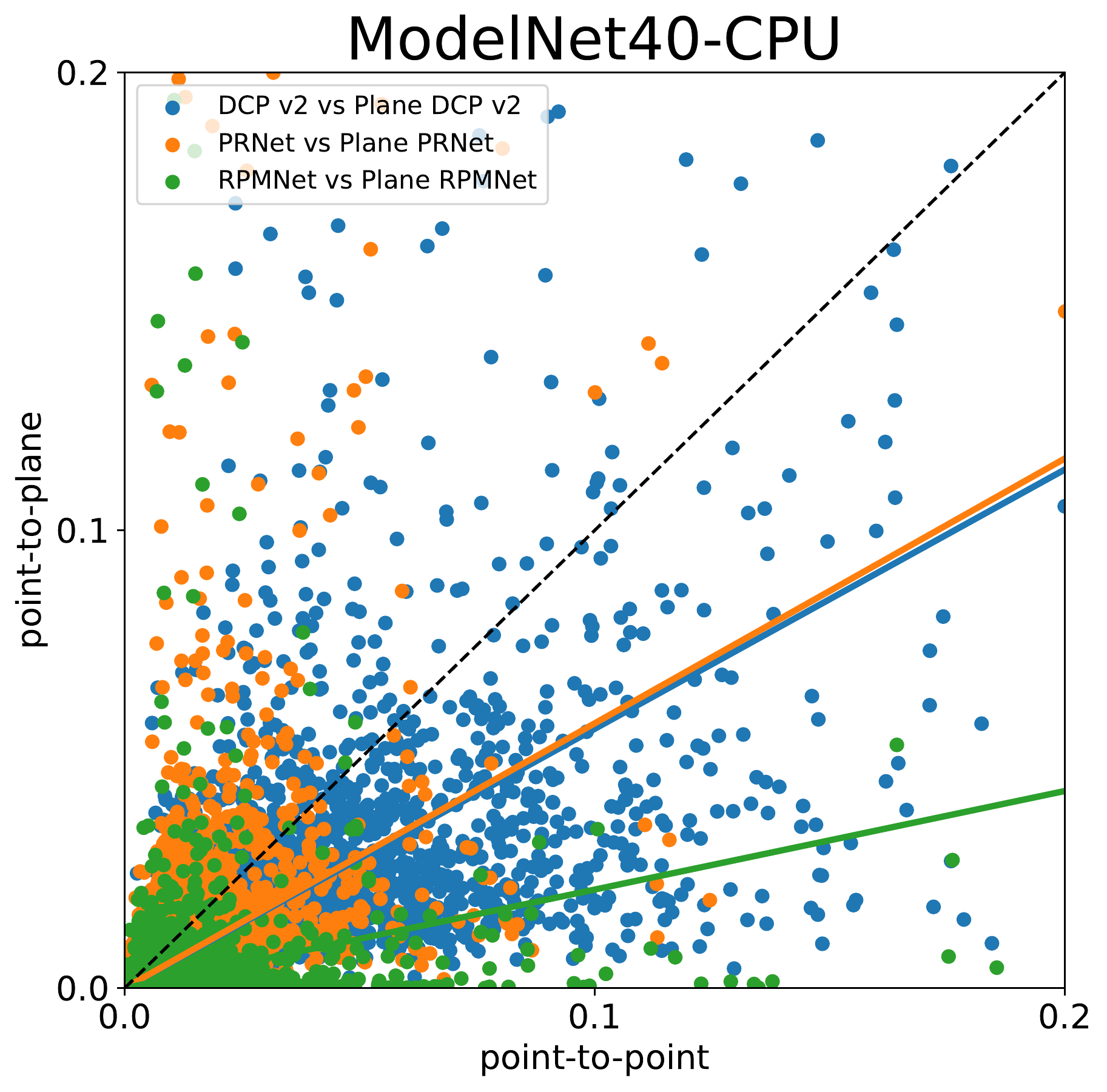} &
    \includegraphics[width=\linewidth]{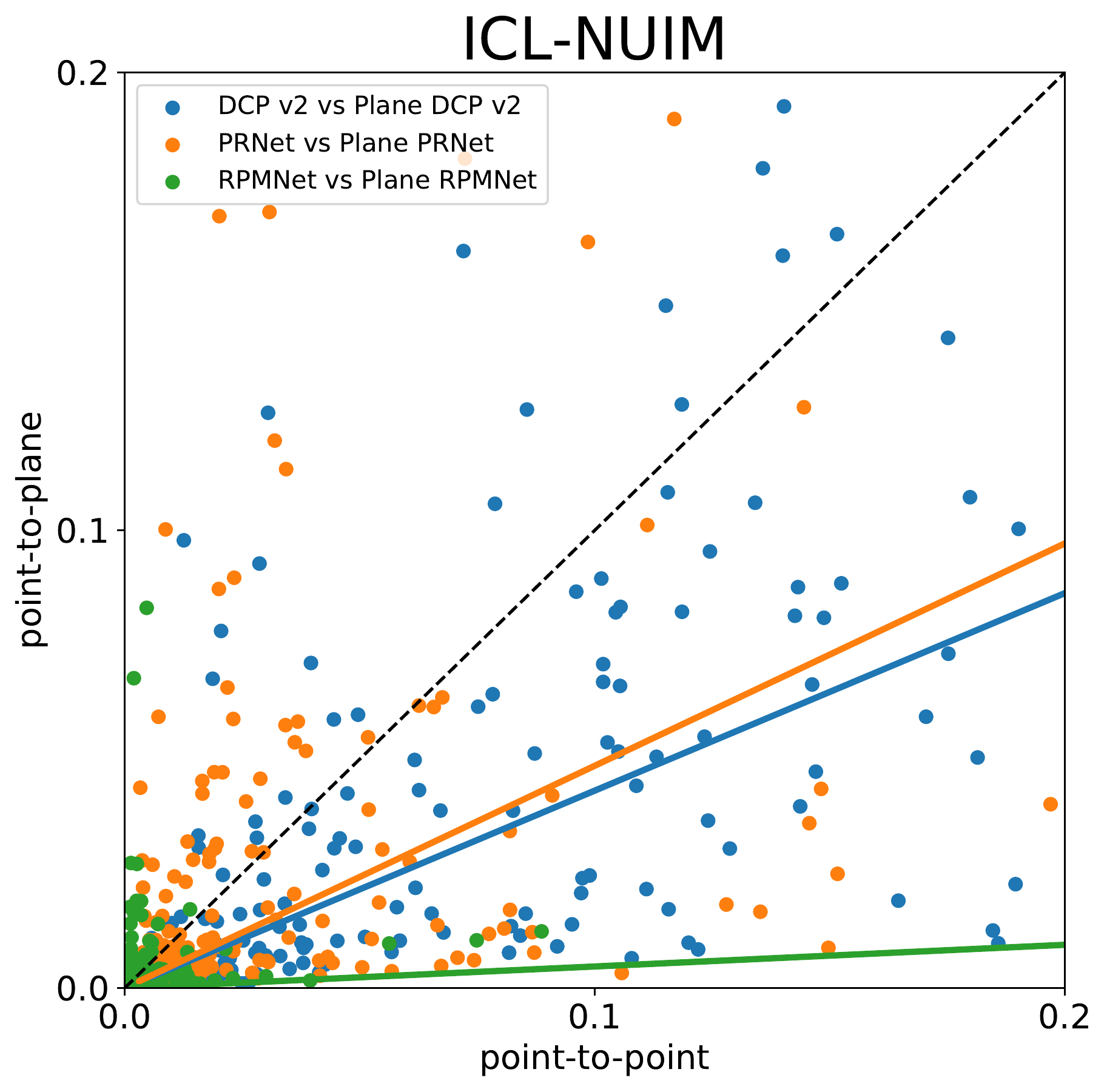} \\
  \end{tblr}
  \caption{Scatter plots for the relationship of performances (i.e., the modified chamfer distances) between the base point-to-point registration method and its point-to-plane extension. The regression line for each pair of base and extension methods is colored equally with those of the dots.}
  \label{fig:point-vs-plane-scatter}
\end{figure}

Second, we conduct a model-level analysis on performance improvement. In \cref{fig:point-vs-plane-scatter}, we represent each point cloud with a dot whose $x$ and $y$ coordinates indicate the chamfer distance for a base point-to-point registration method and its corresponding point-to-plane extension. We also draw regression lines passing through the origin. A regression line corresponds to the plot with the same color. The regression lines located under the black diagonal line mean the better performance of the extensions. On the other hand, these plots for ModelNet40-CPU and ICL-NUIM show a weak correlation of the performances, where the models which are aligned properly by the base methods are not always aligned appropriately by the point-to-plane registration.

Third, we investigate what types of models are more and less favorable for point-to-plane registration. This experiment uses both ModelNet40-CPU and the original ModelNet40 without compositing to clearly show whether point clouds are aligned or not. Since the point clouds, for which the point-to-plane registration obtains better performance, differ depending on the base methods, we show the different point clouds for each base method in \cref{fig:modelnet-cpu-better,fig:modelnet-better}. As shown in these figures, the point-to-plane extensions are often advantageous when the overlapping region of two input point clouds involves a sufficient variety of normal orientations. Although these methods are based on deep learning, this behavior is consistent with the fact reported in the literature on traditional point-to-plane registration~\cite{gelfand2003geometrically}. In addition, the extension for RPMNet performs better than the original RPMNet, particularly when input point clouds have thin structures, such as the foot of an office chair and branches of a plant in \cref{fig:modelnet-better}. Although the performance of the original RPMNet is sufficiently high, the point-to-plane extension makes it more robust to the random displacement of points over the underlying object surface because the point-to-plane registration cares only about the displacement along the point normals.

\begin{figure}[!h]
  \centering
  \includegraphics[width=\linewidth]{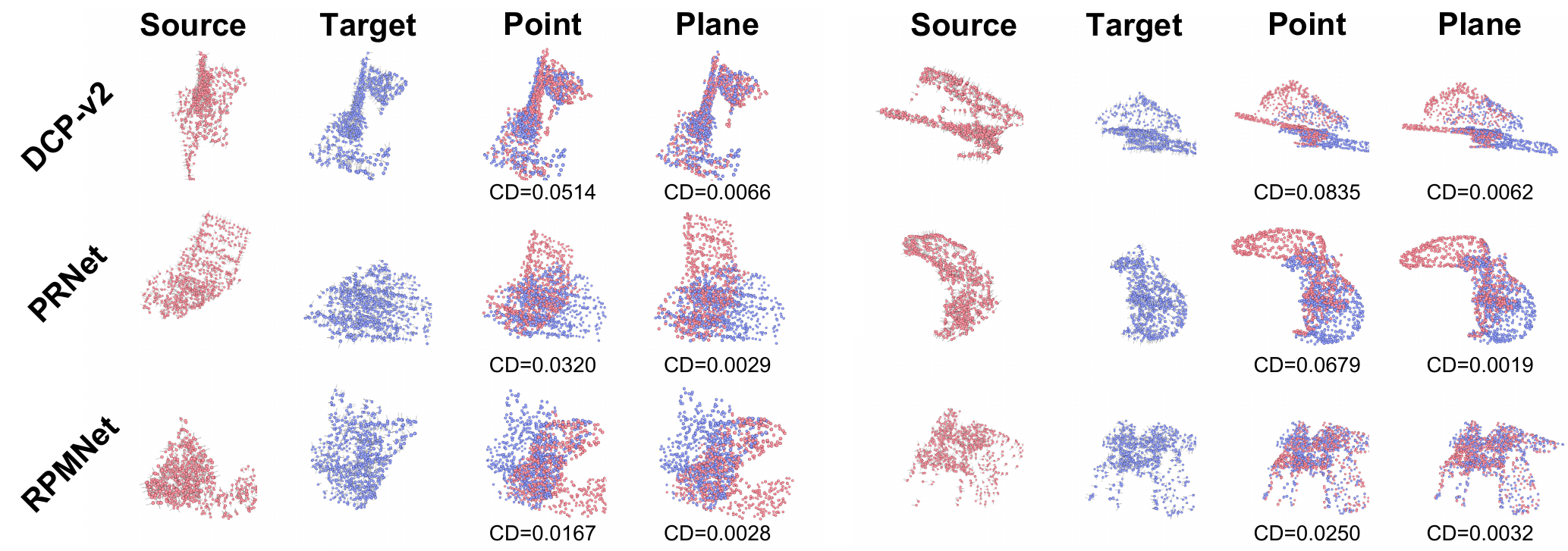}
  \caption{Selected point clouds from ModelNet40-CPU, for which the point-to-plane registration performs better than the base point-to-point registration. The modified chamfer distance (CD) for each result is shown below the aligned point clouds.}
  \label{fig:modelnet-cpu-better}
\end{figure}

\begin{figure}[!h]
  \centering
  \includegraphics[width=\linewidth]{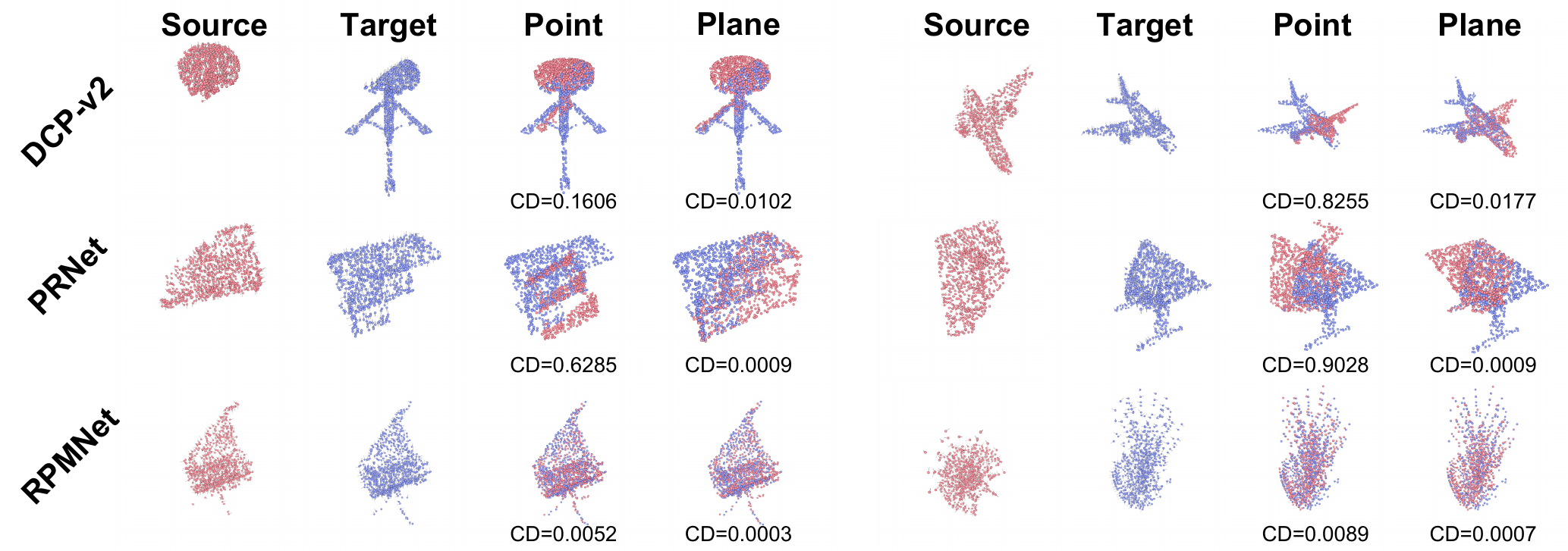}
  \caption{Selected point clouds from the original ModelNet40 without compositing, for which the point-to-plane registration performs better than the base point-to-point registration.}
  \label{fig:modelnet-better}
\end{figure}

We also show the typical failure cases on ModelNet40-CPU and ModelNet40 in \cref{fig:modelnet-cpu-worse,fig:modelnet-worse}, respectively. In these figures, we show the models for which the performances of three point-to-plane extensions are all worse than those of their base methods. What is common in the results for ModelNet40-CPU is that dense points exist at the overlapping region, and the normal vectors of points at almost the same position are inconsistent. In such a case, the neural network may match two points whose normal directions are different, and the subsequent point-to-plane registration can fail. On the other hand, the results for ModelNet40 in \cref{fig:modelnet-worse} demonstrate the limitation of our point-to-plane registration. As noted in the main text, the point-to-plane registration can suffer from rotational and translational ambiguities in the best rigid transformation when the variety of normal directions is insufficient. For example, point clouds in the first and third rows involve rotational ambiguity, while those in the second row involve translational ambiguity.

\begin{figure}[!h]
  \centering
  \includegraphics[width=\linewidth]{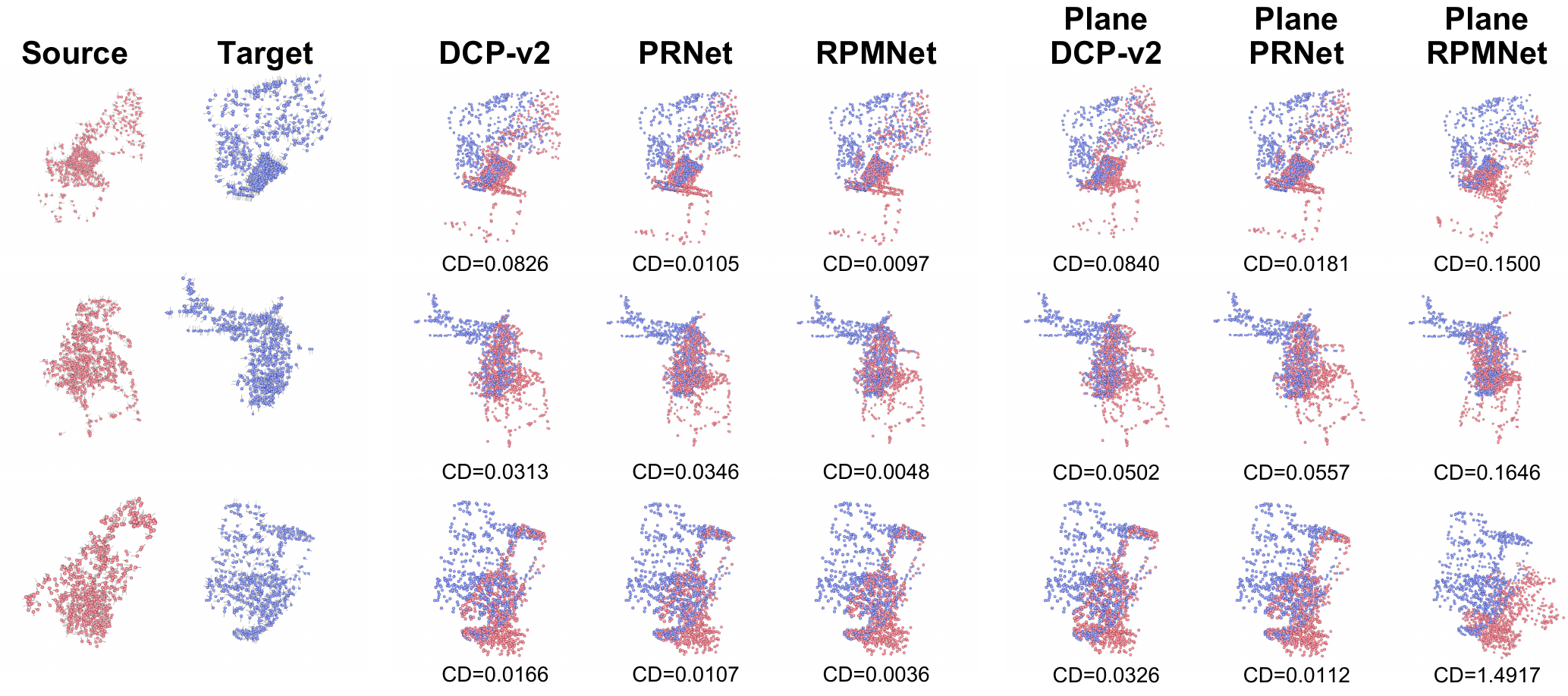}
  \caption{Typical failure cases on ModelNet40-CPU. In these cases, points at the overlapping region of two input point clouds are dense, and those with inconsistent normals are found at almost the same position.}
  \label{fig:modelnet-cpu-worse}
\end{figure}

\begin{figure}[!h]
  \centering
  \includegraphics[width=\linewidth]{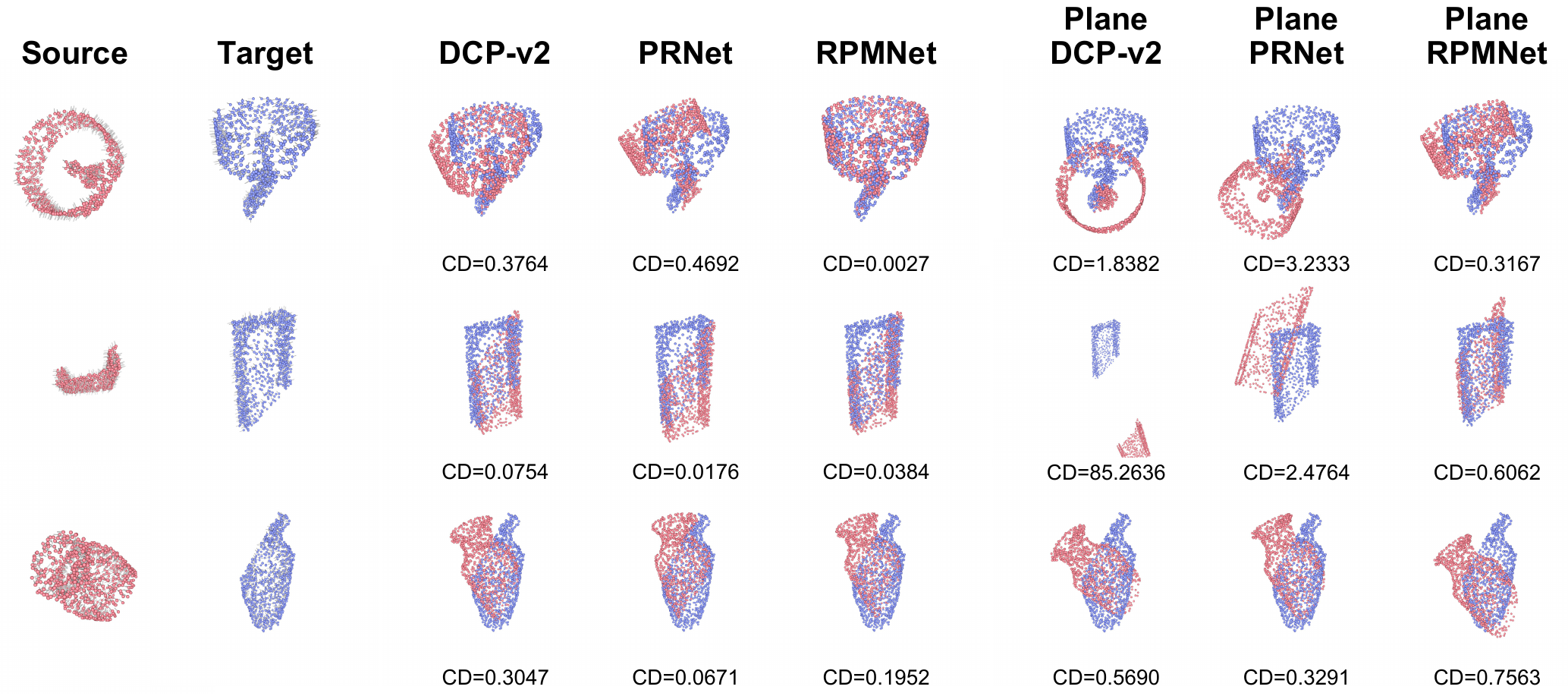}
  \caption{Typical failure cases on the original ModelNet40 without compositing. The results demonstrate a limitation of our point-to-plane extensions, namely that they do not work well for point clouds with rotational and translational ambiguity in the best rigid transformation.}
  \label{fig:modelnet-worse}
\end{figure}

\fi

\end{document}

\end{document}